\RequirePackage{etex}
\documentclass[accepted]{uai2025} 
                        
\usepackage[american]{babel}

\usepackage{natbib} 
\usepackage{mathtools} 
\usepackage{booktabs} 
\usepackage{tikz} 

\usepackage{amsthm,amsmath,amsfonts,amssymb}
\usepackage{xcolor}
\newtheorem{prop}{Proposition}
\newcommand{\dd}{\mathop{}\! \mathrm{d}}
\usepackage{url}
\usepackage{placeins}
\usepackage[ruled,vlined]{algorithm2e}

\usepackage{mathtools}
\usepackage{etextools}

\usepackage{bbm}
\usepackage{multicol}
\usepackage{subcaption}
\usepackage[nice]{nicefrac}

\title{Evasion Attacks Against Bayesian Predictive Models}

\author[1,2]{\href{mailto:<pablo.garcia@icmat.es>?Subject=Your UAI 2025 paper}{Pablo~G.~Arce}{}}
\author[3]{Roi~Naveiro}
\author[1]{David~Ríos~Insua}
\affil[1]{%
    Inst. Math. Sciences, Spanish Nat. Research Council\\
    Madrid, Spain
}
\affil[2]{%
    Universidad Autónoma de Madrid, Escuela de Doctorado\\
    Madrid, Spain
}
\affil[3]{%
    CUNEF Universidad\\
    Madrid, Spain
  }
  
\begin{document}

\maketitle

\begin{abstract}

There is an increasing interest in analyzing the behavior of machine learning systems against adversarial attacks. However, most of the research in adversarial machine learning has focused on studying weaknesses against evasion or poisoning attacks to predictive models in classical setups, with the susceptibility of Bayesian predictive models to attacks remaining underexplored. This paper introduces a general methodology for designing optimal evasion attacks against such models. We investigate two adversarial objectives: perturbing specific point predictions and altering the entire posterior predictive distribution. For both scenarios, we propose novel gradient-based attacks and study their implementation and properties in various computational setups. 
\end{abstract}

\section{INTRODUCTION}

The  ever increasing importance of machine learning (ML) can be appreciated 
 in key relevant societal problems including 
  drug discovery \citep{garcia2025ai} or 
predictive medical systems \citep{corrales2024colorectal}, to name but a few. The impact of large language models and applications based on them has further amplified the transformative potential of ML. Yet, 
  in addition to its positive developments, important misuses have been reported. Many of them arise from attempts by adversaries to outsmart ML algorithms to attain a benefit, leading to the emergence of the relatively recent field of {\em adversarial machine learning} (AML) \citep{joseph}. Most ML models rely on the assumption of independent and identically distributed (iid) data during training and operations. However, in potentially hostile data domains, it is important to be aware that adversaries may modify model inputs
to alter the incumbent distributions. AML aims to provide algorithms that are more robust against adversarial manipulations and focuses on three 
 key issues \citep{riosInsua2023}: studying attacks on algorithms to understand their vulnerabilities; designing defenses to better protect the algorithms against attacks; and providing pipelines that promote best defenses against potential attacks.

AML broadly classifies attacks into two categories: \textit{evasion}, 
occurring during model operations, where adversaries manipulate inputs to cause incorrect predictions without altering the training data; and \textit{poisoning}, where attackers corrupt training data to degrade model performance. Both regression and classification algorithms have been shown vulnerable to these types of adversarial examples \citep{goodfellow2014explaining}, but research has largely focused on classification tasks \citep{cina2023wild,gallego2024protecting},
 most likely because the lack of clear decision boundaries in regression complicates defining and assessing adversarial attacks 
  \citep{gupta2021adv}. Additionally, most research has centered on frequentist models, with vulnerabilities of Bayesian ML approaches remaining much less explored and still under foundational debates 
 \citep{feng2024attacking}.

This paper proposes a novel general framework for designing optimal evasion attacks against Bayesian predictive models in both classification and regression problems. Unlike most previous research, our framework applies not only to attacks targeting predictive means but also to those manipulating predictive uncertainty. Our contributions are therefore:

\begin{itemize}
    \item A novel class of adversarial attacks applicable to general Bayesian predictive models, covering both regression and classification, capable of manipulating any quantity expressed as an expectation under the posterior predictive distribution, including predictive means, expected utilities, and entropies.
    \item A complementary family of attacks that reshapes the full posterior predictive distribution (PPD), allowing adversaries to steer both point predictions and model uncertainty.
    \item A concrete application of these methods to target predictive uncertainty, addressing a largely unexplored area in AML, with notably strong empirical effectiveness.
    \item General schemes for attack generation, requiring only sampling access to the PPD, making them applicable to predictive models where posterior inference is intractable and approximated via Markov Chain Monte Carlo (MCMC) or variational inference (VI).
\end{itemize}

\section{RELATED WORK}

 AML has gained increasing 
attention in recent years \citep{joseph,vorobeichikantar,miller2023adversarial,cina2023wild,riosInsua2023}, as adversaries can manipulate data inputs to achieve malicious goals in high-stakes settings. While early work mainly focused on classification \citep{goodfellow2014explaining}, adversarial vulnerabilities have been demonstrated in a range of learning tasks.


Despite the critical role of predictive uncertainty in high-stakes domains like healthcare and finance, the adversarial robustness of \emph{Bayesian} predictive models remains largely unexplored. While some studies suggest Bayesian methods may inherently resist adversarial attacks \citep{de2021adversarial,yuan2020gradient}, recent findings show that this is not actually guaranteed \citep{feng2024attacking}. Most research evaluates Bayesian robustness against standard attacks that primarily target predictive means, as with projected gradient descent (PGD) or fast gradient sign methods (FGSM). However, attacks specifically tailored to Bayesian models-—particularly those designed to manipulate both predictive means and uncertainty-—are rare. Notable exceptions include methods that require access to ground truth labels during the attack process \citep{galil2021disrupting}, heuristic extensions like PGD+, which alternates PGD steps between attacking predictive means and uncertainty \citep{feng2024attacking} and white-box poisoning attacks that directly target posterior inference rather than predictions \citep{carreau2025poisoning}. Yet, systematic approaches explicitly designed to reshape predictive uncertainty remain scarce. To address this gap, we introduce a general and principled framework for crafting attacks that manipulate both posterior predictive means and uncertainties, providing broad applicability across Bayesian predictive models.

Our framework applies to both classification and regression. While adversarial classification is relatively mature \citep{cina2023wild,gallego2024protecting}, adversarial regression has only recently gained traction, partly because regression lacks natural margins to guide attack and success metrics definition  \citep{gupta2021adv}. Notable contributions include a convex programming framework to generate attacks against regression models with closed-form solutions in numerous cases \citep{balda2018perturbation}; an analysis of several attacks and defenses in the context of autonomous driving \citep{deng2020analysis}, and general methods to create adversarial regression attacks based on the Jacobian of the learned function, tested on industrial datasets \citep{gupta2021adv}. Other efforts focus on poisoning \citep{jagielski2018manipulating,liu2017robust,wen2021great,muller2020data} or outlier attacks \citep{diakonikolas2019sever} in regression contexts. However, none of these proposed attacks are tailored to target Bayesian regression models. Our framework fills this gap and supports sophisticated attacks against general Bayesian predictive models.

\section{PRELIMINARIES}

Consider an agent, referred to as the \textit{predictor} or $P$,
which receives instances with features $ x_i \in \mathbb{R}^p $ and outputs $ y_i $. When $ y_i $ is discrete, we face a classification problem;
 when it is continuous, we deal with a regression problem.
   The data generation process is modeled as follows:
\begin{eqnarray*}
    y_i \vert x_i, \beta, \phi & \overset{\text{iid}}{\sim} & \pi(y_i | f_{\beta} (x_i), \phi)  \\
    \gamma \equiv (\beta, \phi) & \sim & \pi(\gamma)
\end{eqnarray*}
where $\beta$ and $\phi$ are parameters defining the likelihood.
As an example, when $ f_\beta(x_i) = x_i^\top \beta $, with $ \beta = \begin{pmatrix} \beta_0, \beta_1, \dots, \beta_p \end{pmatrix}^\top $, $\phi = \sigma^2$, and $\pi(y_i | f_{\beta} (x_i), \phi)$ is a Gaussian distribution, we recover a linear regression model. 
Similarly, when $f_\beta(x_i) = x_i^\top \beta$ with $\beta = \begin{pmatrix} \beta_0, \beta_1, \dots, \beta_p \end{pmatrix}^\top$, $\phi = 1$, and $\pi(y_i | f_{\beta} (x_i), \phi)$ is a Bernoulli distribution with $p(y_i = 1) = \sigma(f_\beta(x_i))$ where $\sigma(\cdot)$ is the sigmoid function, we recover logistic regression for binary classification.

Given a dataset $ \mathcal{D} = (X_{tr}, y_{tr}) $, where $ X_{tr} \in \mathbb{R}^{n \times p} $ is the design matrix and $ y_{tr}$ is the vector of outputs, the posterior distribution $ \pi(\gamma | \mathcal{D})$ contains all information needed for prediction. With conjugate priors, the posterior distribution has a closed-form solution. However, in general, approximations are typically required to obtain the posterior
using two main types of approaches: optimization-based methods, such as variational inference, and sampling-based methods, such as MCMC, both 
described in \cite{gallego2022current}.

From a predictive standpoint, given a new feature vector $x = (x_1, \dots, x_p)$, the PPD 
$\pi (y | x, \mathcal{D} ) = \int \pi(y | f_{\beta} (x ), \phi)
\pi(\gamma | \mathcal{D}) \dd \gamma $
%
provides all information about $y$ required to make predictions 
  for decision support, given the data $\mathcal{D}$ and covariates $x$. As with the posteriors, in general, we lack a closed-form expression for this distribution, but we can generate samples from it,
either replacing the posterior by (samples from) the variational approximation or based on 
posterior samples generated via MCMC.

In our setting, an additional agent, referred to as the {\em attacker} or $ A $,  seeks to manipulate instances $ (x, y) $ into $ (x', y') = a(x, y) $
  to fool the predictor in search of an objective.
  We focus on scenarios where the attacker modifies only the covariates $ x $, aligning with real-world concerns where adversaries manipulate inputs to achieve favorable outcomes without directly altering the output.
The attacker's objectives vary depending on the specific application.
 We shall consider two main attacks depending on the attacker's structural objective: those targeting 1) point predictions, or 2) the entire PPD.

In the first case, we assume the attacker's interest 
lies in shifting instance covariates $x$ to $x'$ to influence a particular point prediction, written as the expectation 
$\mathbb{E}_{y \mid x, \mathcal{D}} \left[ g(x, y) \right] = \int g(x,y) \pi(y \mid x, \mathcal{D}) \, \mathrm{d} y $
of a function $g(x, y)$ under the PPD.
This framing handles numerous key quantities in predictive modeling, such as the predictive mean, quantiles, or expected utility. In particular, we  
assume that $A$'s goal is to manipulate the covariates $x$ to shift 
  such expectation towards a desired target $G^*$, framed through the problem  
\begin{equation} \label{eq:attack1}
  \min_{x' \in \mathcal{X}} d \left( \mathbb{E}_{y | x', \mathcal{D}} \left[ g(x', y) \right],  G^* \right),
\end{equation}
where $d$ is a distance function and $\mathcal{X}$ is the set of feasible attacks reflecting $A$'s need to avoid detection.  Typically, $\mathcal{X}$ is defined as a norm-bounded constraint of the form $\| x - x' \| \leq \epsilon$, where $x$ represents the untainted data and $\| \cdot \|$ can be any chosen norm.

In the second case, we assume that the attacker's goal is to transform instance covariates $x$ to $x'$ to steer the PPD $\pi (y | x, \mathcal{D} )$ 
 towards a certain target $\pi_A(y)$, referred to as the adversarial PPD (APPD). Again, this 
objective will be framed through the minimization 
  of some distance between the APPD and the actual PPD. 
\section{METHODOLOGY}

We propose a class of general evasion attacks targeting Bayesian predictive models with the two previous goals. Both attack types are 
  initially developed in the white-box setting, assuming the attacker has full knowledge of the predictive model, its prior, and its posterior. Extensions to the gray-box setting are discussed in Section 5.5 and the Supplementary Materials (SM) Sections E and F.3. 



%
\subsection{Attacks Targeting Point Predictions} \label{sec:point}

%
%
%

As mentioned, attacks targeting point predictions will be framed 
  through the optimization problem (\ref{eq:attack1}).  
  Throughout, we use the Euclidean distance as 
   $d(\cdot ,\cdot )$, although other distances could be adopted. 

In simple conjugate settings, e.g. Bayesian linear regression with a normal-inverse-gamma prior, both the posterior and the objective function admit closed-form solutions, allowing for  analytic attack derivations in some cases (see Section \ref{app:normalinversegamma} of SM). In general, however, we lack a closed-form posterior and thus cannot directly compute the optimal attack.  As mentioned, a common practice is to approximate the PPD via sampling.
We can leverage such samples and use stochastic gradient descent (SGD) algorithms for optimization \citep{powell2019unified}. For this, it is necessary to compute unbiased estimates of the gradient of the objective function with respect to the attacked covariates $x'$. 
%
%

Formally, when the data is perturbed to $x'$, 
let the posterior predictive expectation of $g(x',y)$
 be 
\begin{align*}
    \mu(x') \;\equiv\; \mathbb{E}_{y \mid x', \mathcal{D}} \bigl[\,g(x', y)\bigr] 
    \;=\;
    \mathbb{E}_{\gamma \mid \mathcal{D}} 
    \,\mathbb{E}_{y \mid x', \gamma}\bigl[g(x', y)\bigr].
\end{align*}
  Solving problem (\ref{eq:attack1}) with the Euclidean distance 
  is equivalent to 
\begin{equation}\label{eq:general_adv_problem}
    \min_{x' \in \mathcal{X}} 
    J(x') 
    \;\equiv\; 
    \bigl\Vert \mu(x') - G^* \bigr\Vert_2^2.
\end{equation}

\begin{prop}\label{prop:double_sample_gradient}

    Assume (i) $g(x',y)\,\pi(y\mid x',\gamma)$ is measurable and integrable for each $x'$, (ii) $x'\mapsto g(x',y)\,\pi(y\mid x',\gamma)$ is differentiable for almost every $(y,\gamma)$, and (iii) there exists an integrable function $H(y,\gamma)$ with $\|\nabla_{x'}[\,g(x',y)\,\pi(y\mid x',\gamma)\,]\|\le H(y,\gamma)$ for all $x'$. Then the gradient of the objective function in \eqref{eq:general_adv_problem} can be expressed as 
    \begin{align}\label{producto}
        \nabla_{x'} J(x')
        \;=\;
        2\,\bigl(\mu(x') - G^*\bigr)^\top 
        \,\nabla_{x'} \mu(x'),
    \end{align}
    where $\nabla_{x'} \mu(x')$ can be written as
    \begin{equation}\label{eq:point_att}
        \mathbb{E}_{\gamma \mid \mathcal{D}} 
        \,\mathbb{E}_{y \mid x', \gamma} 
        \Bigl[
          \nabla_{x'} g(x', y) 
          \;+\; 
          g(x', y)\,\nabla_{x'} \log \pi(y \mid x', \gamma)
        \Bigr].
    \end{equation}
\end{prop}

\begin{proof}
See SM \ref{app:proof1}.
\end{proof}

\begin{algorithm}[ht]
\DontPrintSemicolon
\SetAlgoLined
\KwIn{
Target value $ G^* $; 
Initial covariate vector $ x \in \mathcal{X} $;
Learning rate $ \eta $;
Feasible set $ \mathcal{X} $;
Maximum iterations $ T $.
}
\KwOut{Optimized covariate vector $ x' $}
\BlankLine
Initialize $ x' \leftarrow x $\;
\For{$ t = 1 $ to $ T $}{
    Sample $ \{ (\gamma^{(i)}, y^{(i)}) \}_{i=1}^N $, where $ \gamma^{(i)} \sim \pi(\gamma \mid \mathcal{D}) $ and $ y^{(i)} \sim \pi(y | x', \gamma^{(i)}) $\;
    Compute $ \widehat{\mu}(x') = \dfrac{1}{N} \sum\limits_{i=1}^N g\left(x', y^{(i)} \right) $\;
    Sample $ \{ (\tilde{\gamma}^{(i)}, y^{(i)}) \}_{i=1}^M $\;
    Compute $ \widehat{\nabla}_{x'} \mu(x') =
        \dfrac{1}{M} \sum\limits_{i=1}^M \Big[ \nabla_{x'} g(x', y^{(i)}) +
          g(x', y^{(i)})\,\nabla_{x'} \log \pi(y^{(i)} \mid x', \gamma) \Big] $\;
    Compute $ \widehat{\nabla}_{x'} J(x') = 2 \left( \widehat{\mu}(x') - G^* \right)^\top \widehat{\nabla}_{x'}  \mu(x') $\;
    Update $ x' \leftarrow x' - \eta \nabla_{x'} J(x') $\;
    Project $ x' \leftarrow \text{Proj}_{\mathcal{X}}(x') $\;
}
\Return $ x' $\;
\caption{White-box point attack}
\label{alg1}
\end{algorithm}

Importantly, the proposition shows that the gradient $\nabla_{x'} J(x')$ is the \emph{product of two expectations} (\ref{producto}). 
Therefore, we can compute an unbiased estimate of the gradient by using sample averages for each factor, employing independent samples for each expectation. Using this estimate, problem \eqref{eq:general_adv_problem} can be solved using projected SGD, as Algorithm \ref{alg1} reflects.

Notice that in some cases, such as Bayesian linear regression with Gaussian noise, we can use a reparameterization trick to express the PPD as two nested expectations, independent of $x'$, greatly simplifying gradient estimation and reducing its variance.

\subsection{Attacks Targeting The Posterior Predictive Distribution}


In this case, we assume that the attacker's goal is 
assessed through the KL divergence between the APPD and the actual PPD,
though other distances may be used. Thus, the attacker's problem is formalized as
\begin{equation} \label{eq:adv_problem_2}
    \min_{x' \in \mathcal{X}} \text{KL} \left( \pi_A(y) \| \pi(y \mid x', \mathcal{D}) \right).
\end{equation}


In general, there will be no analytical expression for the KL divergence in \eqref{eq:adv_problem_2}. 
We then resort again to stochastic gradient-based approaches. This 
 requires an unbiased gradient estimate of \eqref{eq:adv_problem_2} with respect to $ x' $. Notice that minimizing the KL divergence is equivalent to solving
\begin{equation} \label{eq:adv_problem_2_1}
    \min_{x' \in \mathcal{X}}  - \mathbb{E}_y \left[ \log \pi(y \mid x', \mathcal{D}) \right],
\end{equation}
where the expectation is taken with respect to the APPD.

\begin{prop}
Assume (i) $y \mapsto \log \pi(y \mid x', D)\,\pi_A(y)$ is integrable for each $x'$, 
(ii) $x' \mapsto \log \pi(y \mid x', D)$ is differentiable for almost every $y$, 
(iii) there exists an integrable $H(y)$ with $\|\nabla_{x'} \log \pi(y \mid x', D)\|\le H(y)$ for all $x'$, 
and (iv) the map $\gamma \mapsto \pi(y \mid x', \gamma)\,\pi(\gamma \mid D)$ is integrable for each $x'$, with $\nabla_{x'} \pi(y \mid x', \gamma)$ dominated by an integrable function. 
Then, the gradient of the objective function in \eqref{eq:adv_problem_2} can be expressed as
\[
-\,\mathbb{E}_y \Biggl[
  \frac{\mathbb{E}_{\gamma \mid D}[\nabla_{x'}\,\pi(y \mid x',\gamma)]}
       {\mathbb{E}_{\gamma \mid D}[\pi(y \mid x', \gamma)]}
\Biggr].
\]
\end{prop}

\begin{proof}
See SM \ref{app:proof2}.
\end{proof}

\begin{algorithm}[htb]
\DontPrintSemicolon
\SetAlgoLined
\KwIn{
    Initial covariate vector $ x \in \mathcal{X} $;
    Adversarial PPD $ \pi_A(y) $;
    Data $ \mathcal{D} $;
    Learning rate $ \eta $;
    Feasible set $ \mathcal{X} $;
    Maximum iterations $ T $;
    Number of samples per iteration $ R $;
    Sequence $ \{ M_\ell \} $ (e.g., $ M_\ell = M_0 2^\ell $);
    Weights $ \{ \omega_\ell \} $ (e.g., $ \omega_\ell \propto 2^{-\tau \ell} $).
}
\KwOut{Optimized covariate vector $ x' $}
\BlankLine
Initialize $ x' \leftarrow x $\;
\For{$ t = 1 $ to $ T $}{
    Sample $ y \sim \pi_A(y) $\;
    Independently sample $ R $ levels $ \ell^{(1)}, \dots, \ell^{(R)} $ with probabilities $ \omega_{\ell} $\;
    Compute $ \Delta g_{x', \ell^{(r)}}(y) $ for each sampled level $ \ell^{(r)} $ using the antithetic construction:\;
    \quad If $ \ell^{(r)} = 0 $, compute $ \Delta g_{x', 0}(y) = g_{x', M_0}(y) $\;
    \quad If $ \ell^{(r)} > 0 $, compute $ \Delta g_{x', \ell^{(r)}}(y) = g_{x', M_{\ell^{(r)}}}(y) - \dfrac{1}{2} \left( g^{A}(y) + g^{B}(y) \right) $\;
    \quad 
    \tcp{$ g^{A}(y) $ and $ g^{B}(y) $ are computed using subsets $ A $ and $ B $ of size $ M_{\ell^{(r)} -1} $ from the $ M_{\ell^{(r)}} $ samples}
    \;
    Compute the unbiased gradient estimate:\;
    \quad $ \widehat{\nabla}_{x'} J(x') = \dfrac{1}{R} \sum\limits_{r=1}^{R} \dfrac{ \Delta g_{x', \ell^{(r)}}(y) }{ \omega_{\ell^{(r)}} } $\;
    Update $ x' \leftarrow x' - \eta \widehat{\nabla}_{x'} J(x') $\;
    Project $ x' \leftarrow \text{Proj}_{\mathcal{X}}(x') $\;
}
\Return $ x' $\;
\caption{White-box distribution attack}
\label{alg2}
\end{algorithm}

Given the nested expectations in its expression, a standard sampling-based approximation of the gradient, with all expectations estimated 
via sample averages, would be biased. Yet it is possible to construct an unbiased multi-level Monte Carlo (MLMC) gradient estimator, using a debiasing technique adapted from \cite{goda2022unbiased}.
 To wit, for a randomly chosen $y \sim \pi_A(y)$, let us define
\begin{align*}
    g_{x'}(y) \equiv - \frac{\mathbb{E}_{\gamma | \mathcal{D}} \left[ \nabla_{x'} \pi(y | x', \gamma )\right] }{\mathbb{E}_{\gamma | \mathcal{D}} \left[ \pi(y | x', \gamma )\right]}.
\end{align*}
Clearly, $\mathbb{E}_y \left[g_{x'}(y) \right]$
is equal to the gradient of the objective function in \eqref{eq:adv_problem_2_1}. Define also 
\begin{align*}
    g_{x', M}(y) \equiv - \frac{ \frac{1}{M}  \sum_{m=1}^M \nabla_{x'} \pi(y | x', \gamma_m )}{\frac{1}{M}  \sum_{m=1}^M  \pi(y | x', \gamma_m )}, 
\end{align*}
for a sample $\left \{\gamma_m \right\}_{m=1}^M \overset{\text{iid}}{\sim} \pi(\gamma | \mathcal{D})$ and consider a 
 positive increasing sequence $\{ M_\ell \}$, with $M_\ell \rightarrow \infty$ as $\ell \rightarrow \infty$. The strong law of large numbers ensures that $g_{x', M}(y)$ converges almost surely to $g_{x'}(y)$ and, thus, we can write the gradient of the objective function in \eqref{eq:adv_problem_2_1} as
 \begin{equation} \label{eq:mlmc}
     \mathbb{E}\left[g_{x'}(y) \right] = \lim_{\ell \rightarrow \infty } \mathbb{E}\left[g_{x', M_{\ell}}(y)  \right] = 
    \sum_{l=0}^\infty \mathbb{E}[\Delta g_{x', \ell}(y)],
 \end{equation}
where the random variables $\Delta g_{x', \ell}(y)$ are defined such that $\mathbb{E}[\Delta g_{x', 0}(y)] = \mathbb{E}[g_{x', M_0}(y)]$ and $\mathbb{E}[\Delta g_{x', \ell}(y)] = \mathbb{E}[g_{x', M_{\ell}}(y) - g_{x', M_{\ell - 1}}(y)]$ for $\ell > 0$.

The idea is then to create an unbiased estimate of the gradient by sampling indices $\ell^{(1)}, \dots, \ell^{(R)}$ of the sum in \eqref{eq:mlmc} with probabilities $\omega_{\ell^{(1)}}, \dots, \omega_{\ell^{(R)}}$ and computing
\begin{equation*}
    \frac{1}{R} \sum_{r=1}^R \frac{\Delta g_{x', \ell^{(r)}}(y)}{\omega_{\ell^{(r)}}}.
\end{equation*}

In practice, the $\omega_{\ell}$ are chosen to be $\propto 2^{-\tau \ell}$ for a given $\tau$, and the coefficients $\Delta g_{x', \ell}(y)$ are estimated using the antithetic coupling approach in \cite{goda2022unbiased}. In our experiments, we set $\tau$ heuristically to balance cost and variance. Algorithm \ref{alg2} compiles the required computations.  To better assess the practicality of the proposed attacks, we provide an 
 analysis of their computational overhead in Appendix~\ref{app:comp}. This overhead can be substantially reduced through various strategies, including more efficient sampling techniques, low-rank approximations, and adaptive sampling schedules.


\section{EXPERIMENTS}

This Section assesses the effectiveness of the proposed attacks
through {\em security evaluation plots} (SEP) as in \cite{gallego2024protecting}, 
displaying performance metrics at different attack intensities,
defined through the maximum $\epsilon$ allowed for the $L_2$ norm  of the difference between the original and perturbed covariates,
that is ${\cal X}=\{ x' : ||x-x'||_2  \leq \epsilon \}$. 
 Since, to the best of our knowledge, no existing attacks offer the same level of generality as ours, we construct a tailored benchmark attack. Inspired by FGSM \citep{goodfellow2014explaining}, it estimates the objective’s gradient once (via sampling with our algorithms) and perturbs each covariate by $\pm \epsilon$ along the negative gradient direction.
To account for attack stochastic variation, each experiment is repeated several times, and results are reported as mean plus/minus two standard errors. Code to reproduce all experiments as well as full hyperparameter specifications 
   are available at \url{https://github.com/pablogarciarce/AdvReg}. Computations were performed on an 88 \textit{Intel(R) Xeon(R) E5-2699 v4} CPUs server with 250 GB RAM.

The first experiment provides structural insights with a simple case 
in Bayesian linear regression with synthetic data. We then consider
three examples with real data, still within the linear regression 
realm, as well as a case with a Bayesian neural network (BNN) model. Finally, we explore attacks in classification, specifically targeting predictive uncertainty.

\subsection{ATTACKS ON BAYESIAN REGRESSION WITH SYNTHETIC DATA}

To build intuition about attack performance and behavior,
we begin with a simple simulation study using a linear regression model with two covariates. The dataset consists of 1,000 observations with iid $\mathcal{N}(0, 1)$ covariates. The response variable is generated 
as $ y \sim \mathcal{N}(\beta_0 x_0 + \beta_1 x_1, \sigma^2) $, with $ \beta_0 = -1$, $ \beta_1 = 2$, and  noise variance $ \sigma^2 = 1$. For model evaluation, a 70-30\% train-test split is applied. 
Conjugate normal priors $ \beta_0, \beta_1 \sim \mathcal{N}(0, 1) $
 are used. Under this setup, the PPD for any 
$ x $ follows a normal distribution with mean $ x^\top \mu_n $ and variance $ x^\top \Lambda_n^{-1} x + \sigma^2 $, with $ \mu_n $ and $ \Lambda_n $ respectively being the posterior mean and covariance matrix for $ \beta $, SM Section B.
    
We start by evaluating attacks targeting point predictions. Particularly, the attacker's goal is to shift the mean of the PPD towards a specified value, chosen 
as $ \mu^* = 3 $. In this scenario, the objective function has a closed-form solution, enabling the computation of exact gradients. Furthermore, an analytical solution to the attacker's problem is available (SM \ref{app:point_analytical}). To ensure the validity of the proposed strategy, let us solve the problem using SGD with gradients estimated through sampling, as in Algorithm \ref{alg1}. 
SM \ref{app:unb_toy} presents an empirical evaluation of the unbiasedness of the gradient estimates produced by the algorithm.

%

Through this simple model, we can visualize the impact of the attack on test instances. Figure \ref{subfig:point_field_indep} portrays covariates $ x_1 $ and $ x_2 $ on the $x$- and $y$-axes, respectively. Each attack is represented by an arrow with the tail corresponding to the untainted covariate vector, and the head indicating the perturbed one. Components of $ \mu_n $ are depicted as a blue arrow. All arrows are parallel and 
aligned with the  $ \mu_n $  direction, pointing out in the direction in which $ \mu_n^T x $ increases. This is expected, as the initial PPD  mean is lower than the target $ \mu^* $ for all test data points, aligning with the analytical solution \eqref{mean_att_l2} in the SM.

\begin{figure}[ht]
\begin{multicols}{2}
    \includegraphics[width=\linewidth]{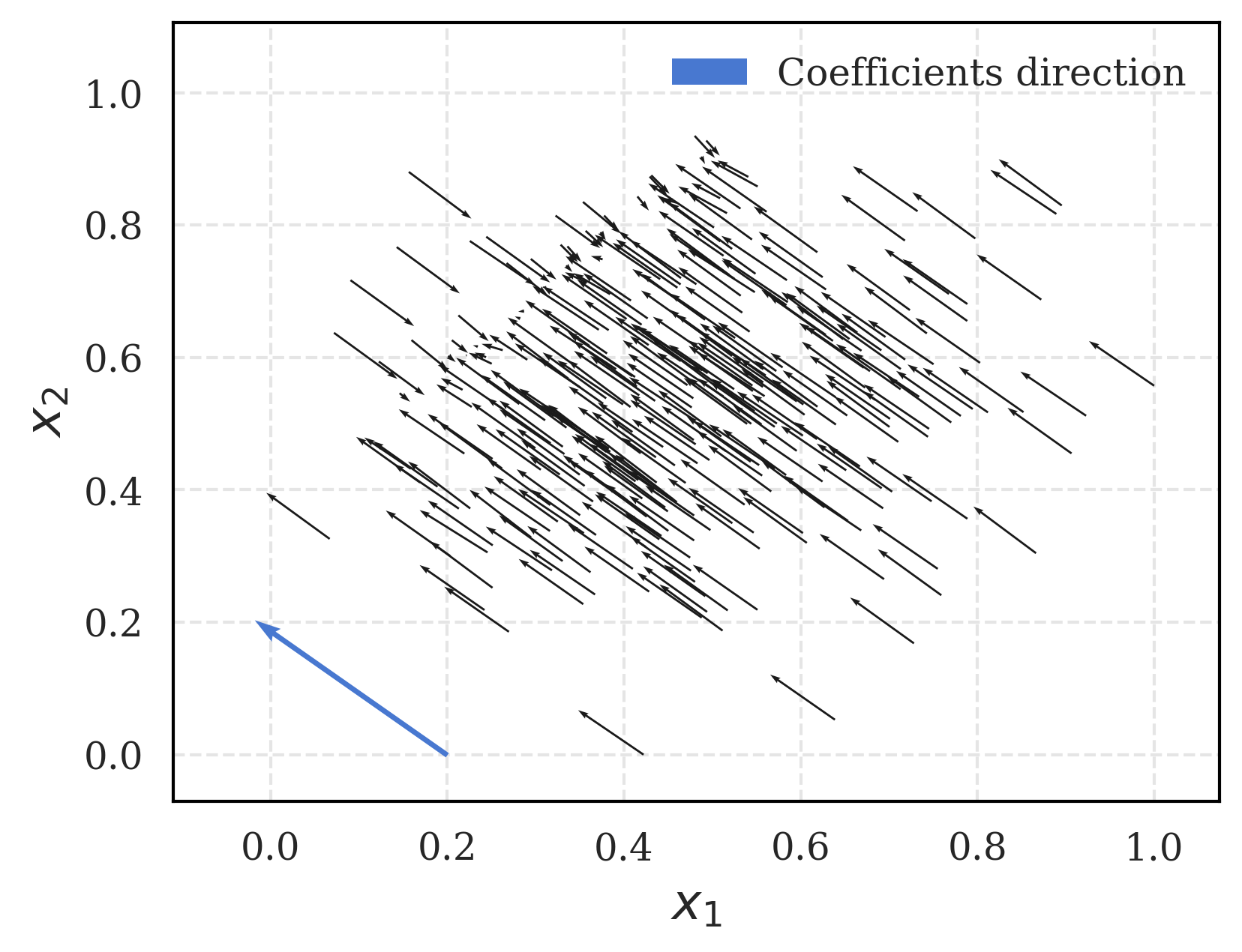}
    \subcaption{Perturbation field.}\label{subfig:point_field_indep}\par 
    \includegraphics[width=\linewidth]{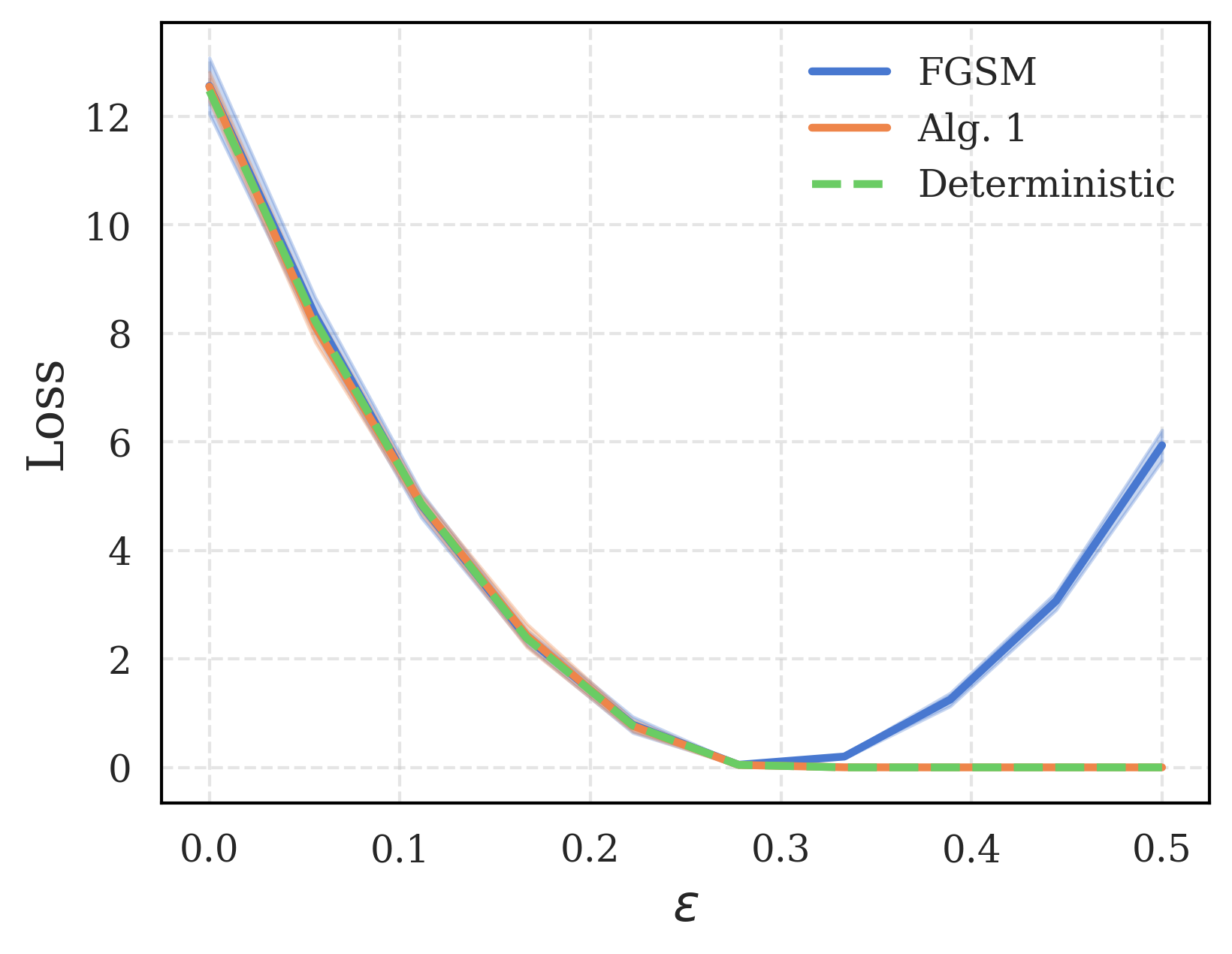}
    \subcaption{SEP.}\label{subfig:point_eps_indep}\par 
    \end{multicols}
\caption{Attacks to the PPD mean.}
\label{fig:point_indep}
\end{figure}

Figure \ref{subfig:point_eps_indep} presents the SEP for three attack strategies: the deterministic approach in SM \ref{app:point_analytical}, Algorithm \ref{alg1}, and the FGSM-inspired attack. The attacker's goal is to shift the predictive mean of an instance, initially $\mu=-0.5$, towards the target value $ \mu^*$.
The plot shows the squared difference between $ \mu^*$ and the PPD mean induced by the attack, for attack intensities varying from 0 to 0.5. 
The method based on Algorithm \ref{alg1} closely matches the deterministic solution across all attack intensities. In this 
 simple setting, FGSM's performance is comparable to the other two approaches for $\epsilon \leq 0.3$. Notably, a perturbation with $ \Vert x - x' \Vert_2 \simeq 0.2 $ is sufficient to increase the PPD mean from $-0.5$ to $1.9$, with the target mean being achieved at $ \Vert x - x' \Vert_2 \simeq 0.3 $.

Let us turn now our attention to attacks targeting the full PPD,
  which, in the linear regression setting considered with know variance, follows a normal distribution. The attacker’s goal is to increase predictive uncertainty at a specific point $x$. This is achieved by selecting a normal APPD with the same mean as the original PPD but with,
    say, four times the variance. Under simulated data, with large-sample size, 
    predictive uncertainty is primarily driven by aleatoric uncertainty (due  to inherent randomness or variability in the system), as epistemic uncertainty (arising from incomplete knowledge about the system studied) diminishes with an increased number of data points
     \citep{banks2022adversarial}. Hence, since aleatoric uncertainty cannot be influenced by manipulating covariates, such attacks are unlikely to be effective in this context. To better evaluate the impact of attacks targeting posterior uncertainty, we generate a synthetic dataset with smaller sample size $n=10$, where epistemic uncertainty has a more pronounced effect. Given that both the original and adversarial PPDs are normal, there exists a closed-form solution for the KL divergence. However, we approximate the attacks using SGD with gradient estimates obtained via sampling as in Algorithm \ref{alg2}. An empirical evaluation of the algorithm's gradient estimates in SM \ref{app:unb_toy} confirms their unbiasedness.

Figure \ref{subfig:distr_field_n10} illustrates how a set of points is altered, with arrows representing changes as in Figure \ref{subfig:point_field_indep}. Most arrows point in directions that move away from the training data, where predictive variance increases.
\begin{figure}[ht]
\begin{multicols}{2}
    \includegraphics[width=\linewidth]{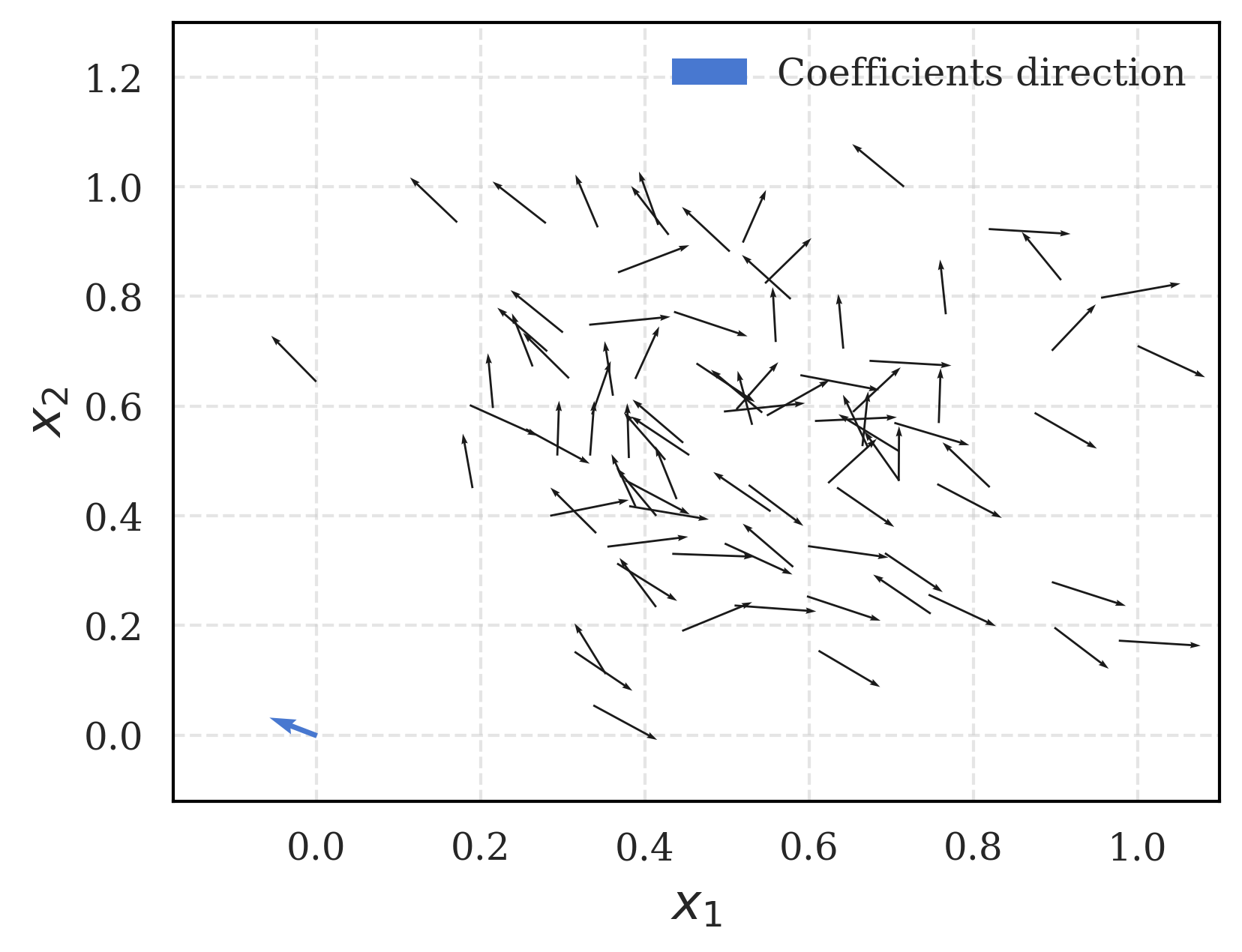}
    \subcaption{Perturbation field.}\label{subfig:distr_field_n10}\par 
    \includegraphics[width=\linewidth]{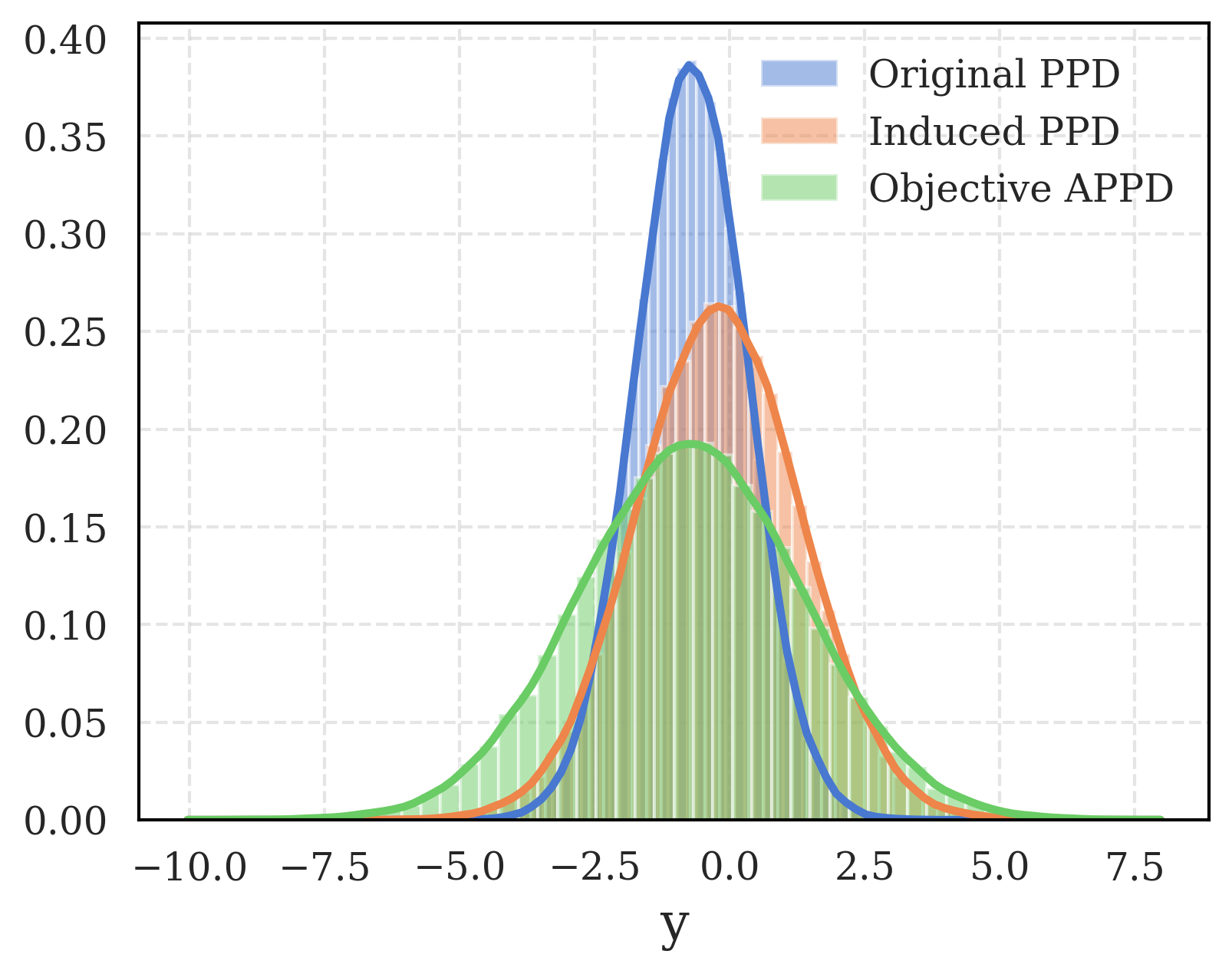}
    \subcaption{Original, adversarial and induced PPDs.}\label{subfig:distr_m2s_n10}\par 
    \end{multicols}
\caption{Attacks to full PPD.}
\label{fig:perturbations}
\end{figure}
%
Figure \ref{subfig:distr_m2s_n10} shows the PPD, APPD, and induced PPD for an $\epsilon = 2$ attack targeting an initial    $x=(0.43, 0.33)$. Interestingly, even an attack with $\epsilon = 1$ was not capable of fully shifting the PPD towards the APPD. This is because the selected point $x$ lies within the training data distribution, where predictive variance is low due to reduced epistemic uncertainty. Doubling the variance would require the attacker to significantly move the covariates beyond the training data distribution, where epistemic uncertainty is higher.


\subsection{ATTACKS ON BAYESIAN REGRESSION WITH REAL DATA}

This Section evaluates the effectiveness of attacks on three real-world regression datasets. In all cases a linear regression model with normal-inverse gamma prior is used. 

The \textbf{Wine Quality} dataset \citep{wine_data} contains 4,898 white wine samples described by 11 physicochemical features; the response variable is wine quality, scored on an ordinal scale from 3 to 8. The \textbf{Energy Efficiency} dataset \citep{athanasios_tsanas_energy_2012} 
comprises 768 samples with 8 input features related to building characteristics; the selected response is heating load. Finally, the \textbf{California Housing} dataset  \citep{cal_housing_data} 
includes 20,640 samples, with 8 input features, and the response 
representing the median house value for California districts.
In all cases, the attacker's goal for attacks targeting point predictions is to drive the mean of the PPD towards twice the mean value of the response variable over the training set; for full PPD
 attacks, the goal is to drive the PPD towards a normal distribution with twice the mean and four times the variance of the untainted PPD.

Table \ref{table:datasets_point_RMS} presents the root mean squared error (RMSE) between the induced posterior predictive mean and its adversarial target on test instances for each dataset, under attack intensities 0.0 (no attack), 0.2, and 0.5. Higher attack intensities consistently move the induced mean closer to the adversarial target.
In particular, for the California Housing dataset, even a low-intensity attack ($\epsilon = 0.2$) is highly effective, reducing the average deviation from 2.27 to 0.36. This highlights that even small manipulations of the covariates can significantly shift point predictions, underscoring the model’s vulnerability to targeted attacks. Additionally, Table~\ref{table:datasets_point_RMSE_R} in SM \ref{app:three} illustrates  how high-intensity attacks increase the RMSE on test set predictions.

\begin{table}[htbp]
\centering
\begin{tabular}{lccc}
\hline
\textbf{Dataset} & \textbf{0.0} & \textbf{0.2} & \textbf{0.5} \\
\hline
\textbf{Wine}    & $5.94 \pm 0.13$   & $4.07 \pm 0.25$  & $1.39 \pm 0.43$ \\
\textbf{Energy}  & $5.96 \pm 0.59$   & $4.94 \pm 0.58$  & $3.52 \pm 0.53$ \\
\textbf{Housing} & $2.27 \pm 0.13$   & $0.36 \pm 0.09$  & $0.025 \pm 0.003$ \\
\hline
\end{tabular}
\caption{RMSE for point attacks.}
\label{table:datasets_point_RMS}
\end{table}

For attacks targeting the full PPD, Table \ref{table:datasets_ppd} reports mean and two standard deviation intervals of the KL between the APPD and the PPD induced by the attack across all test instances, for each dataset and attack intensity. As before, higher attack intensities result in PPDs that more closely resemble the target, further demonstrating the attack effectiveness in manipulating predictive uncertainty.

\begin{table}[htb]
\centering
\begin{tabular}{lccc}
\hline
\textbf{Dataset} & \textbf{0.0} & \textbf{0.2} & \textbf{0.5} \\
\hline
\textbf{Wine}    & $20.02 \pm 1.33$   & $11.31 \pm 0.83$  & $8.41 \pm 1.15$ \\
\textbf{Energy}  & $31.51 \pm 5.78$   & $24.44 \pm 4.82$  & $15.84 \pm 3.46$ \\
\textbf{Housing} & $4.00 \pm 0.61$    & $0.82 \pm 0.24$   & $0.89 \pm 0.25$ \\
\hline
\end{tabular}
\caption{KL divergence for full PPD attacks.}
\label{table:datasets_ppd}
\end{table}

    
    
    

\subsection{REGRESSION CASE STUDY WITH A BAYESIAN NEURAL NETWORK}\label{sec:wine}

This Section demonstrates the framework’s scalability by increasing the complexity of the attacked predictive model. 
We use again the UCI wine dataset.
Imagine a winery controls the \emph{actual} vector of quality indicators $x$ when producing a wine, but may \emph{declare} a slightly altered vector $x'$ to a third--party rating agency.  
The market value of a bottle is taken to be $g(x, y) = \exp\left(\nicefrac{y^2}{100}\right) - c(x)$ where $c(x)$ denotes the production cost.  The exponential term captures how small quality gains translate into disproportionately higher willingness to pay once a wine is perceived as {\em reserve}  rather than mid-range.

As the quality $y$ given indicators $x$ is uncertain, it is modeled with a Bayesian neural network (BNN) with a single three neuron hidden layer.  
Normal priors are assigned to the weights and biases in both the hidden and output layers. The response is assumed to follow $y \sim \mathcal{N}(bnn(x), \sigma^2)$, where $bnn(x)$ represents the BNN output for the covariate vector $x$, with $\sigma^2$ 
  assigned a $\Gamma(2, 2)$ prior. Posterior sampling is conducted using the NUTS algorithm \citep{hoffman2014no}.

If the winery were to report the true indicators $x$, the expected profit would be $\mathbb E_{y\mid x,\mathcal D}[g(x,y)]\approx 1.4$.  
Instead, it reports $x'$, indistinguishable in routine checks,    artificially inflating the agency’s valuation to a target of $G^{\ast}=3$ while respecting
$\lVert x-x'\rVert_{2}\le\epsilon$.  
Formally, the winery solves problem~\eqref{eq:attack1}.  
 Since the objective is an intractable posterior expectation, we employ Algorithm~\ref{alg1}.

We calculated the squared difference between the predictive value of the wine under tainted indicators and the target value, for several intensities $ \epsilon $. The corresponding SEP is   in SM Section \ref{app:wine}. Notably, observe that a strategically chosen perturbation with $ \epsilon = 0.5 $ is sufficient to reach the desired target value using Algorithm \ref{alg1}. 
In contrast, the FGSM  attack requires a higher $\epsilon$ and exhibits more instability in achieving the same objective.

To assess the impact of the norm used to constrain perturbations, we generated attacks under $L_2$ and $L_1$ constraints with $\epsilon=0.3$. The original and perturbed covariates are in Table \ref{tab:perturbations} SM. As expected, the $L_1$-constrained attack alters only a subset of covariates, emphasizing the need to safeguard influential features as a potential defense strategy.

In a final experiment, the attacker's goal is to make the 
agency believe that the mean quality of the chosen wine is higher than it truly is, while also increasing the uncertainty about its quality. 
\begin{figure}[h!]
\centering
    \includegraphics[width=0.49\linewidth]{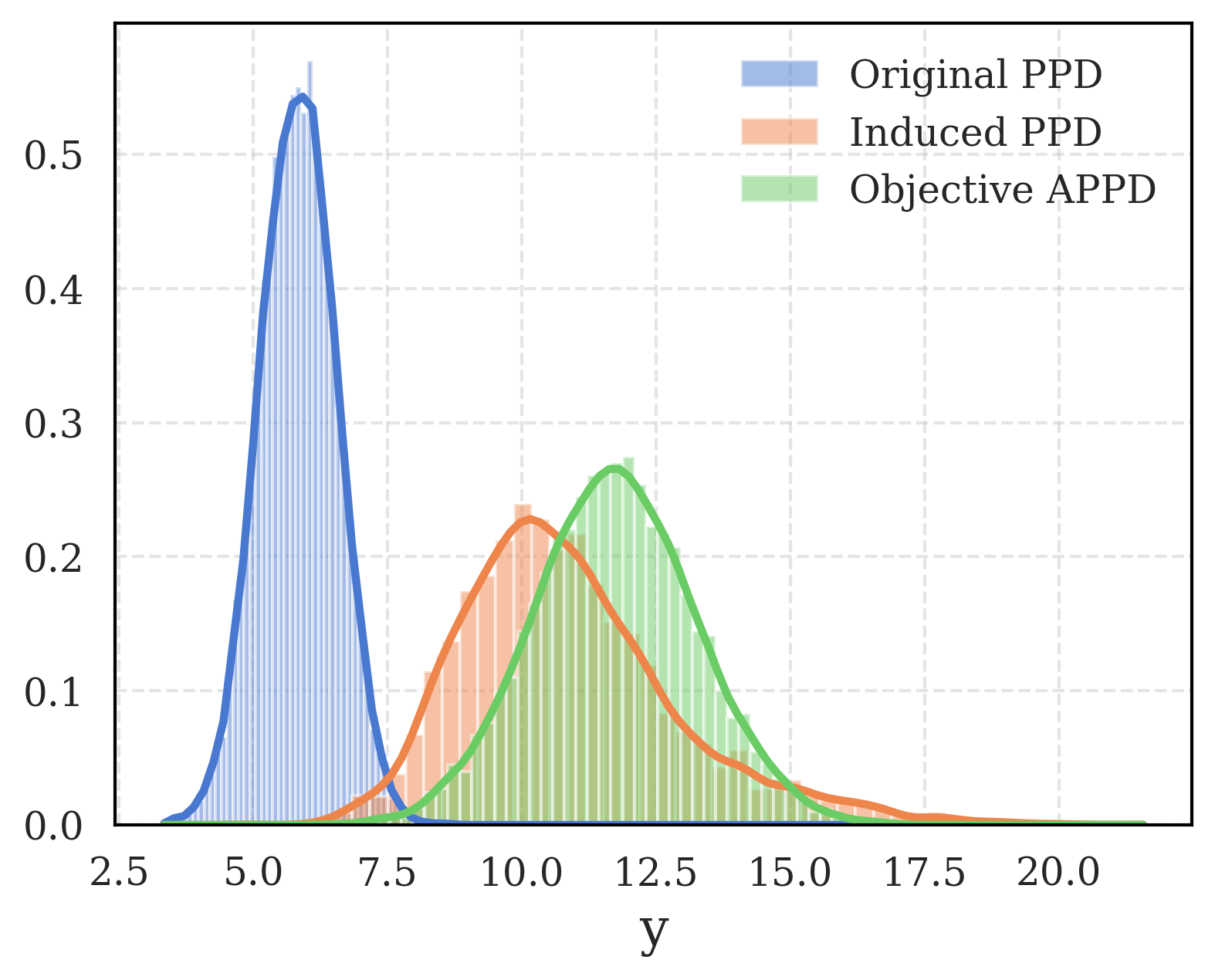}
\caption{Original, adversarial and induced PPDs.}\label{subfig:wine_PPD}
\end{figure}
This is exemplified by setting the APPD to be normally distributed with double the mean and four times the variance of the original PPD for that wine. Figure \ref{subfig:wine_PPD} illustrates the original, target, and induced PPD with $\epsilon=0.7$. This attack causes a significant shift in the posterior predictive mean and a notable increase in variance. Under the original covariates, the probability that the wine’s quality exceeds 7.5 was nearly zero, but now it exceeds 0.9. This could easily lead to wrong marketing decisions, highlighting the potential real-world impact of such adversarial attacks. 

\subsection{ATTACKING UNCERTAINTY IN CLASSIFICATION}

\begin{figure*}[h]
\begin{multicols}{4}
    \includegraphics[width=\linewidth]{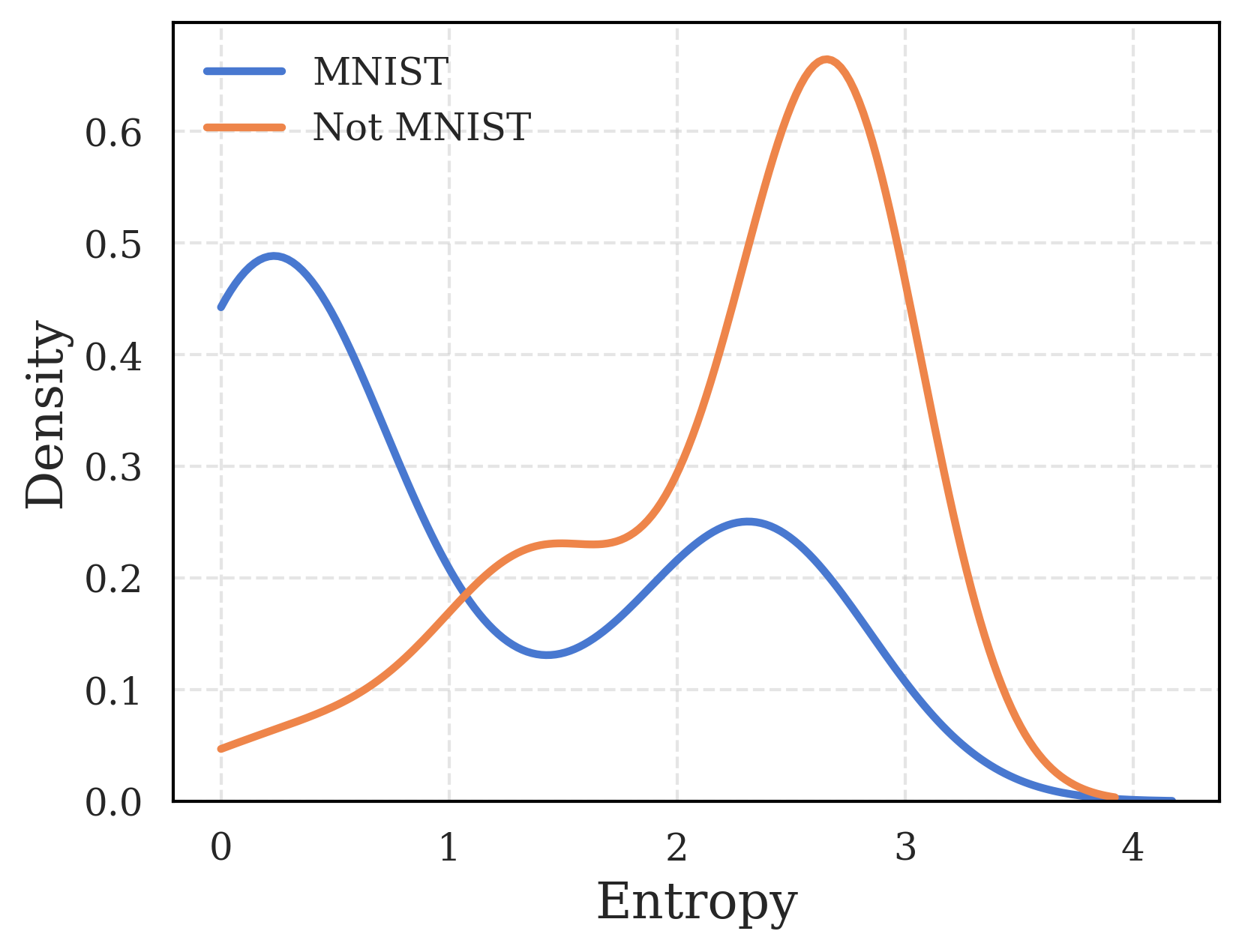}
    \subcaption{PPD entropy over a sample of test dataset for MNIST and notMNIST data.}\label{fig:baseline_entropy}\par 
    \includegraphics[width=\linewidth]{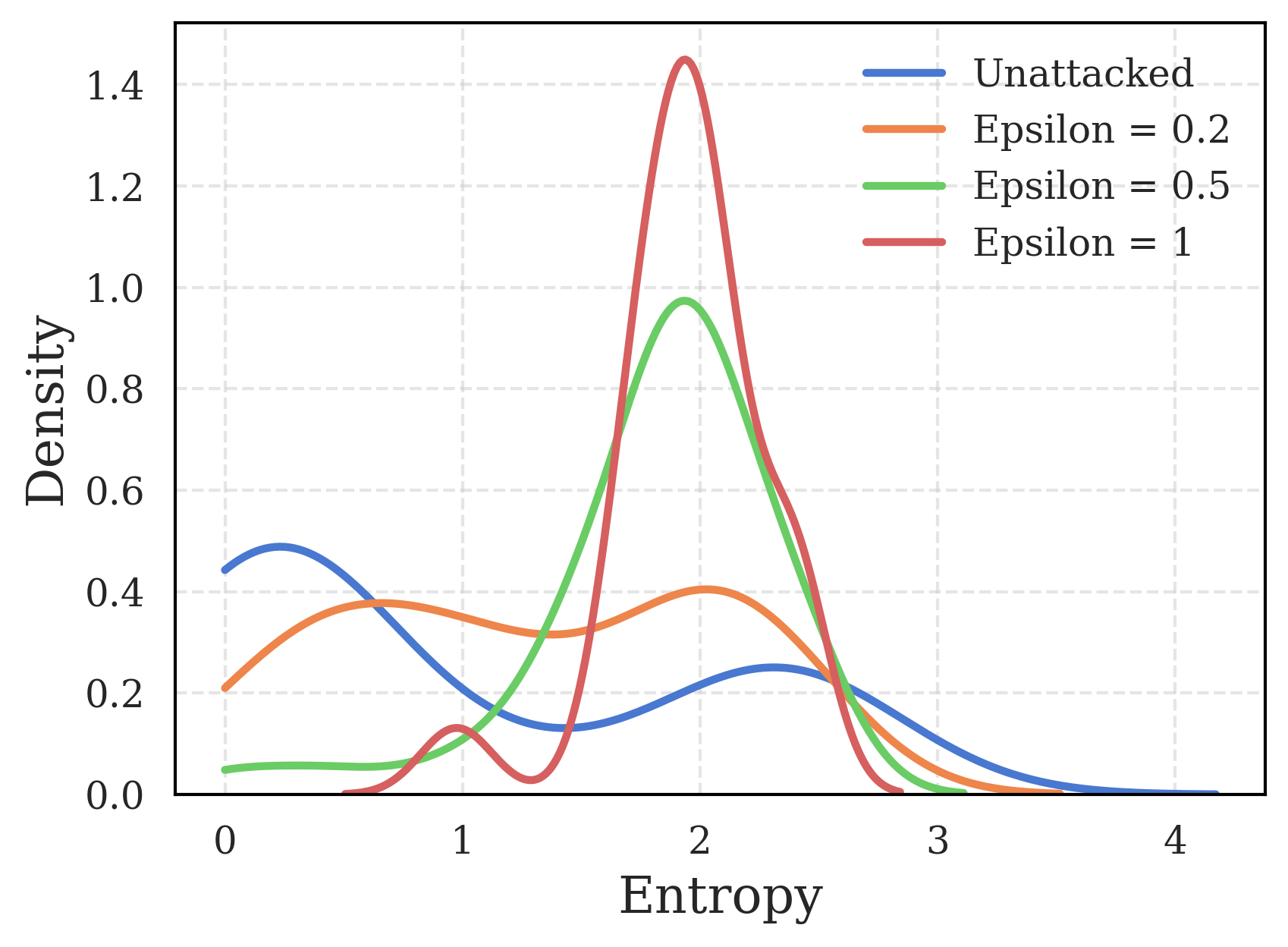}
    \subcaption{Inflating the entropy for MNIST data.}\label{fig:point_rise}\par
    \includegraphics[width=\linewidth]{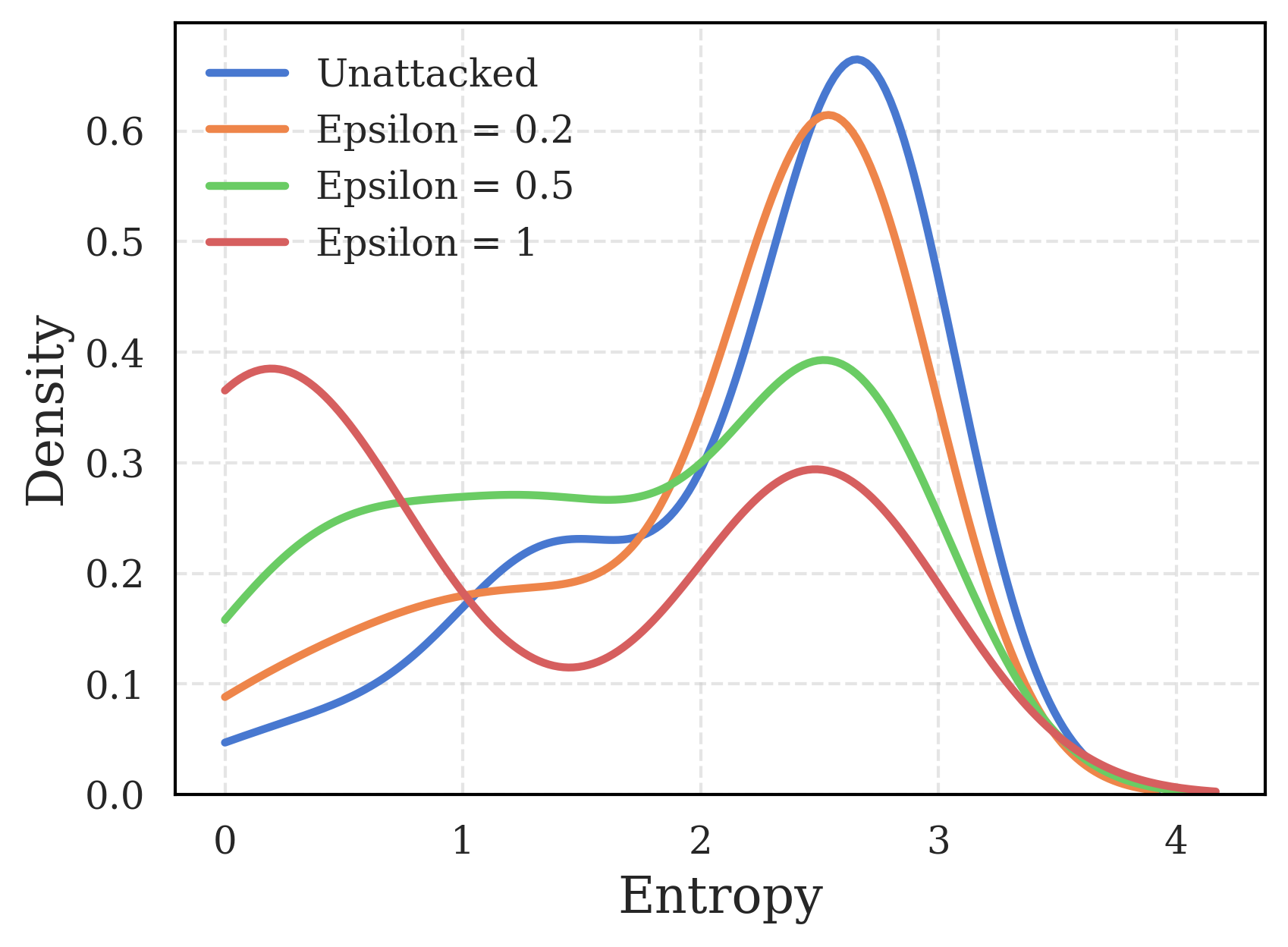}
    \subcaption{Lowering the entropy for notMNIST data.}\label{fig:point_lower}\par
    \includegraphics[width=\linewidth]{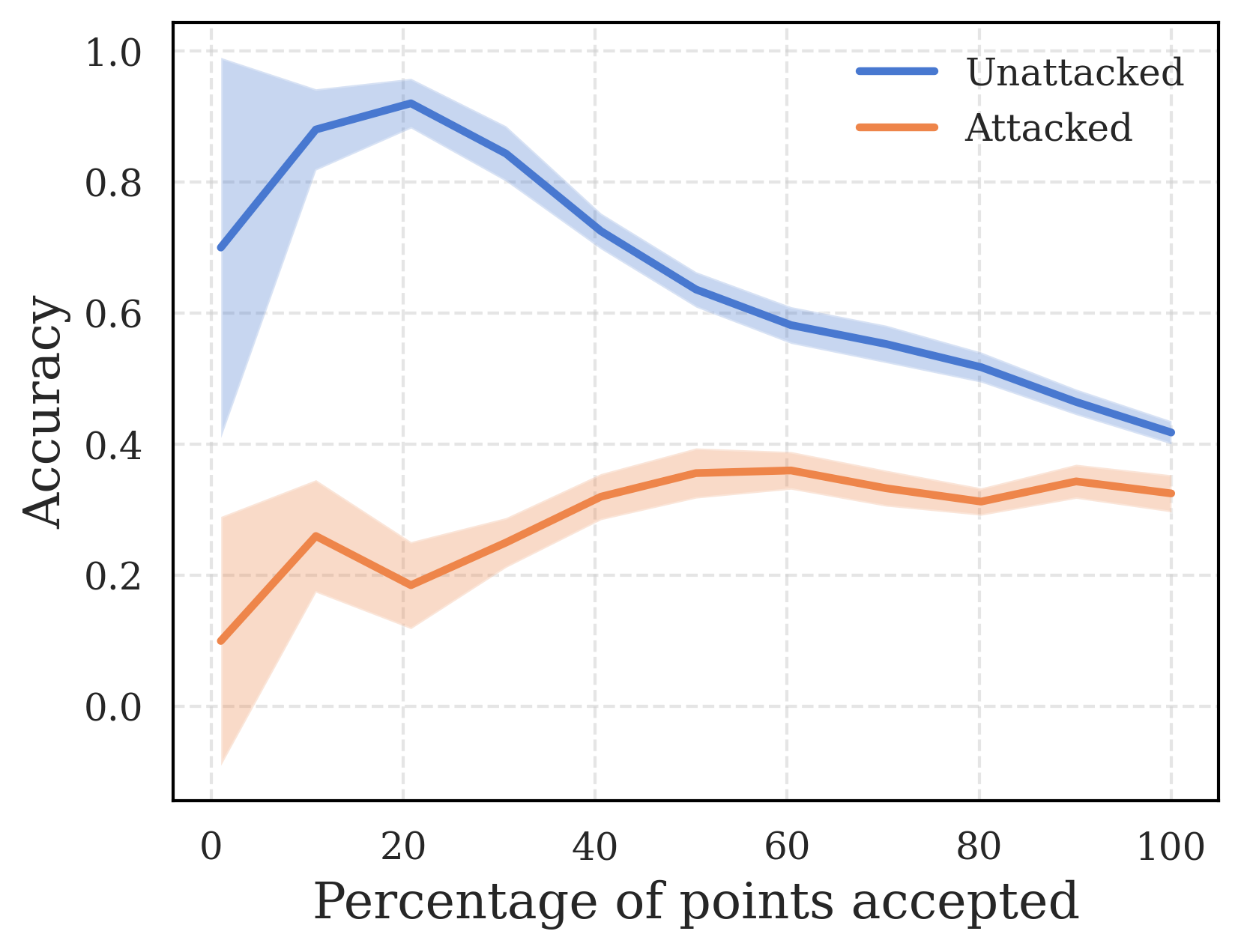}
    \subcaption{Selective prediction accuracy.}\label{fig:acc_reject}
    \end{multicols}
\caption{Attacks to the expected entropy over the PPD.}
\label{fig:point_ent}
\end{figure*}

This Section applies the proposed attacks to classification, targeting posterior uncertainty in a BNN trained on the MNIST handwritten digits dataset \citep{mnist} using variational inference. 
The label $y$ of an image with pixel values $x$ is predicted  using a BNN with two hidden layers, each containing ten neurons. Normal priors on the weights and biases of both layers are 
 employed, and the response follows a categorical distribution with class probabilities based on the BNN's output. A low-rank multivariate normal distribution is used as variational family.  

A key objective of this experiment is to evaluate whether posterior uncertainty, either under clean or attacked data, can help detect out-of-distribution (OOD) from in-distribution (ID) samples. To this end, we test our attacks on both the MNIST test set (ID) and the notMNIST dataset \cite{bulatov2011notmnist} (OOD), which consists of images of letters A to J, structured similarly to MNIST. For each attack, we analyze the distribution of predictive entropy in both datasets, following \cite{lakshminarayanan_simple_nodate}. Under clean data, if uncertainty measures are effective for OOD detection, we expect MNIST test set instances to exhibit low entropy, while notMNIST instances should show higher entropy. This is 
  showcased in Figure \ref{fig:baseline_entropy}, which presents a density plot of predictive entropy for both datasets.


The goal of these attacks is to manipulate predictive entropy, effectively disrupting the model's expected behavior. Specifically, we aim at reducing or eliminating the distinction between ID and OOD entropy estimates. To achieve this, 
 we introduce two attack strategies, both constrained within an $L_2$-norm $\epsilon$-ball centered on the clean image. To quantify the impact of our attacks, we analyze the shift in the entropy distribution across varying perturbation intensities $\epsilon$.

First, we conduct point attacks targeting predictive entropy, aiming to increase entropy for MNIST images and decrease it for notMNIST. These attacks naturally align with the framework outlined in Section \ref{sec:point}. For instance, to maximize predictive entropy, we define $ g(x, y) = y $, where $ y $ is the one-hot encoding of the predicted class, and set the target as $G^* = \frac{1}{p} \mathbf{1}_p$, where $ p $ is the number of classes and $\mathbf{1}_p$ is a $p$-dimensional vector of ones.

Figure \ref{fig:point_rise} illustrates the results for the attack inflating predictive entropy on MNIST images, while Figure \ref{fig:point_lower} corresponds to the one aiming to reduce it in the notMNIST images. As Figure \ref{fig:point_rise} shows, attacks with $\epsilon = 0.5$ are highly effective at increasing predictive entropy for MNIST test images, making the BNN nearly useless for OOD detection purposes. Conversely, Figure \ref{fig:point_lower} shows that while stronger perturbations progressively lower the predictive entropy of notMNIST images, this attack proves more challenging than increasing uncertainty. Notably, an attack with $\epsilon = 0.5$ nearly eliminates the distinction in predictive entropy between ID and OOD samples.

Figure \ref{fig:acc_reject} examines the impact of attacks on a balanced test set of MNIST and notMNIST samples. The experiment simulates an uncertainty-based filtering approach, where a fraction of samples with the lowest predictive entropy is retained, and accuracy is computed on this subset. The $x$-axis represents the percentage of retained samples, while the $y$-axis shows the corresponding classification accuracy. In  absence of an attack, the model effectively utilizes predictive uncertainty, achieving high accuracy by primarily selecting ID samples. 
  After the attack, accuracy drops sharply across all thresholds, indicating that the model no longer distinguishes between ID and OOD samples based on uncertainty.

 Consider now PPD attacks, where the adversarial perturbation is optimized to align the PPD with a high-entropy target distribution (a categorical distribution assigning equal probability to every class) for MNIST data, or a low-entropy target (a deterministic distribution) for notMNIST data. Figures \ref{fig:distr_rise} and \ref{fig:distr_lower} illustrate the effects of these attacks. As in the previous case, inflating entropy proves easier than reducing it. Notably, even small perturbations induce significant shifts in uncertainty estimates, severely undermining the reliability of the model’s predictive uncertainty for OOD detection. Further results, including SEPs for both attack methods, FGSM-inspired approximations, and attacks on deep ensembles, are in Section \ref{app:mnist} of the SM.

Together, these experiments demonstrate the susceptibility of VI-based BNNs to adversarial perturbations targeting uncertainty. This underscores the need for robust defense mechanisms in probabilistic classification settings, ones that extend beyond traditional attacks on posterior means (such as PGD and FGSM) to counter more sophisticated uncertainty-targeting attacks like those here introduced.

\begin{figure}[ht]
\begin{multicols}{2}
    \includegraphics[width=\linewidth]{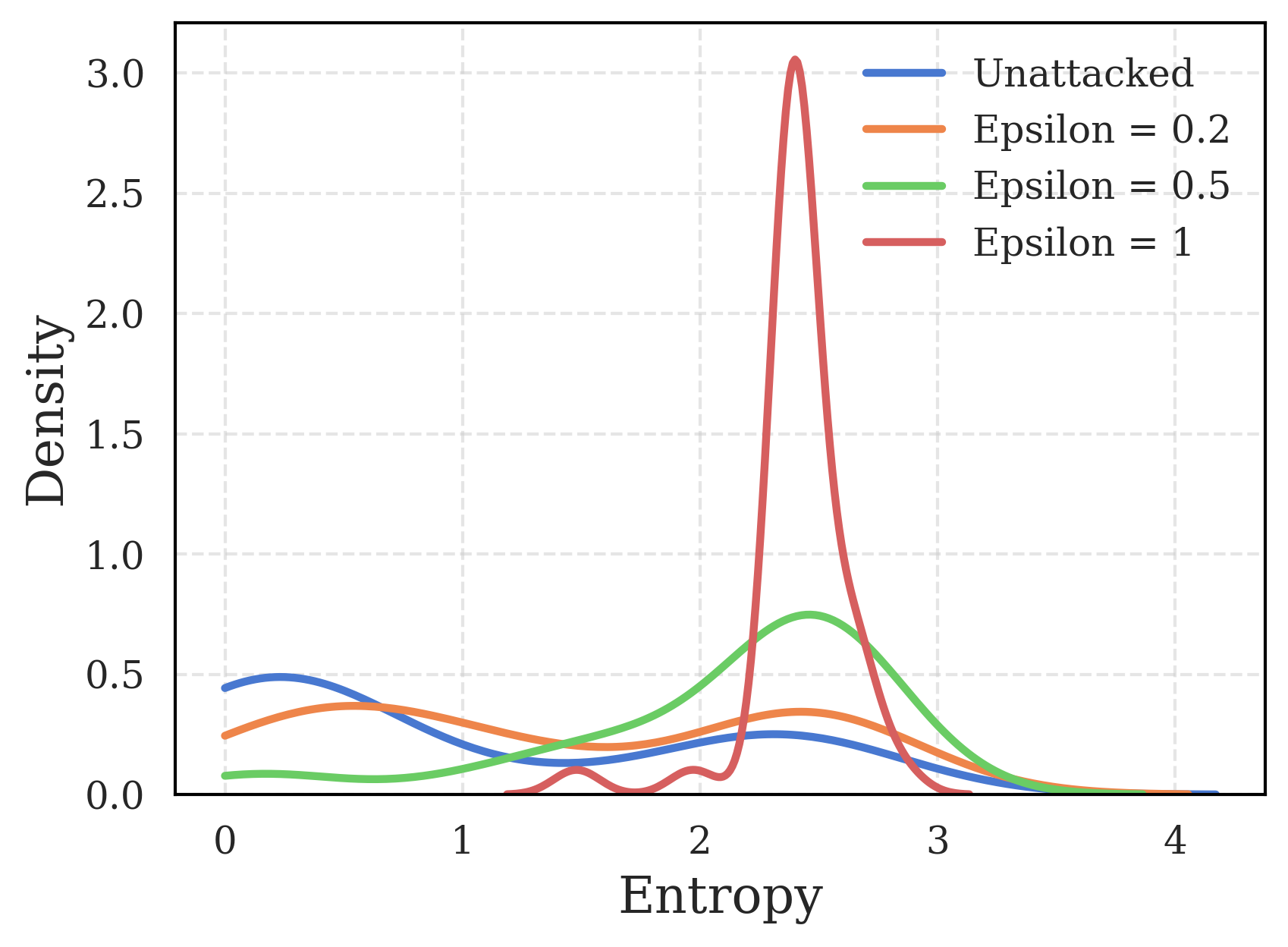}
    \subcaption{Inflating the entropy for MNIST data.}\label{fig:distr_rise}\par  
    \includegraphics[width=\linewidth]{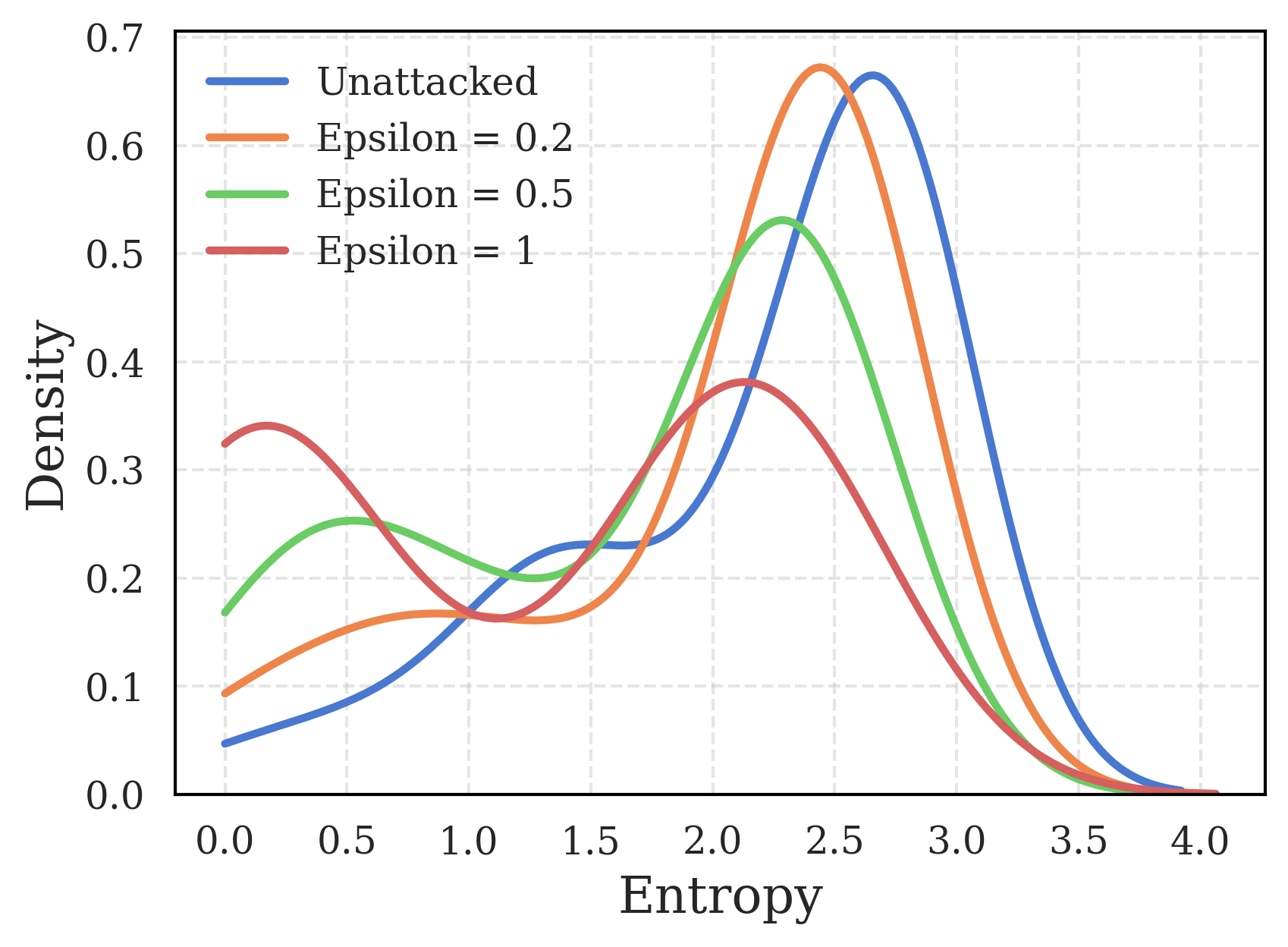}
    \subcaption{Lowering the entropy for notMNIST data.}\label{fig:distr_lower}\par 
    \end{multicols}
\caption{Attacks to full PPD.}
\label{fig:distr_entrop}
\end{figure}

\subsection{GRAY-BOX ATTACKS}

This Section explores the transferability of our attacks under more constrained information scenarios. We conduct experiments using the setup in Section~\ref{sec:wine}, specifically, under the following assumptions:

\begin{enumerate}
    \item \textbf{Unknown Architecture}: The attacker has no knowledge of the defender’s specific BNN architecture. Attacks are thus computed based on one architecture but evaluated against a different one.

    \item \textbf{Limited Training Data}: The attacker has access only to one-third of the training dataset. Attacks constructed using this reduced dataset are tested against a defender model trained on the full dataset.

    \item \textbf{Partial Feature Knowledge (Best Predictors)}: 
    The attacker knows only the 7 most predictive features (based on predictive power) out of the 11 available features. Attacks are computed assuming those 7 predictors but evaluated against a defender using all 11 features.

    \item \textbf{Partial Feature Knowledge (Worst Predictors)}: Same as the previous scenario, but this time the attacker has access only to the 7 least predictive features.
\end{enumerate}

For each setting, we run the attack procedure four times with different random initializations, where each initialization involves small perturbations around the initial point used for stochastic gradient descent. We then report the mean and standard deviation of the resulting performance metric.

Consider first attacks targeting the PPD mean. 
 As in Section~\ref{sec:wine}, let us focus on the case in which the defender aims to modify the wine indicators to shift the expected market value of the produced wine, given by $g(x, y) = \exp\left(\nicefrac{y^2}{100}\right)$, towards a target value $G^*$. Figure~\ref{fig:rmse_graybox} represents the RMSE between the achieved expected value and the desired adversarial target value $G^* = 3$. The results indicate that adversaries with knowledge only of the dataset (without access to the model architecture) achieve performance nearly identical to those conducting full white-box attacks, with both approaches attaining optimal targeting at $\epsilon = 0.50$. Similarly, attacks utilizing only the 7 most predictive features maintain strong effectiveness (RMSE = 0.43 at $\epsilon = 0.50$), significantly outperforming attacks restricted to the 7 least predictive features (RMSE = 1.36). Notably, even when restricted to the least predictive features, attacks remain reasonably effective, although less so than those leveraging more informative features. As anticipated, limited data access (1/3 of the dataset) reduces attack success (RMSE = 0.73 at $\epsilon = 0.50$), though the degradation is less severe than might be expected given the substantial reduction in available information.

\begin{figure}[ht]
\begin{multicols}{2}
    \includegraphics[width=\linewidth]{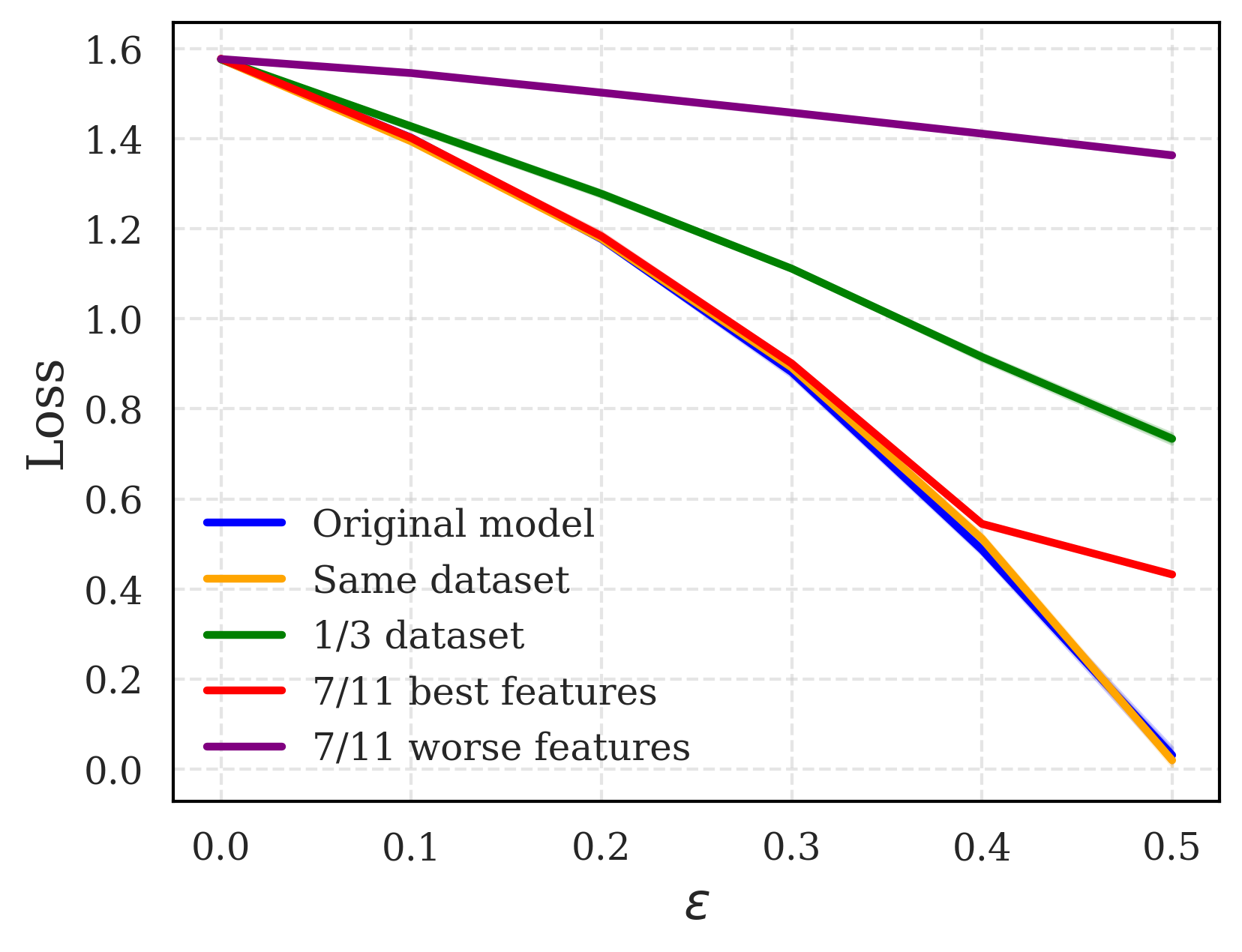}
    \subcaption{Point attacks.}\label{fig:rmse_graybox}\par  
    \includegraphics[width=\linewidth]{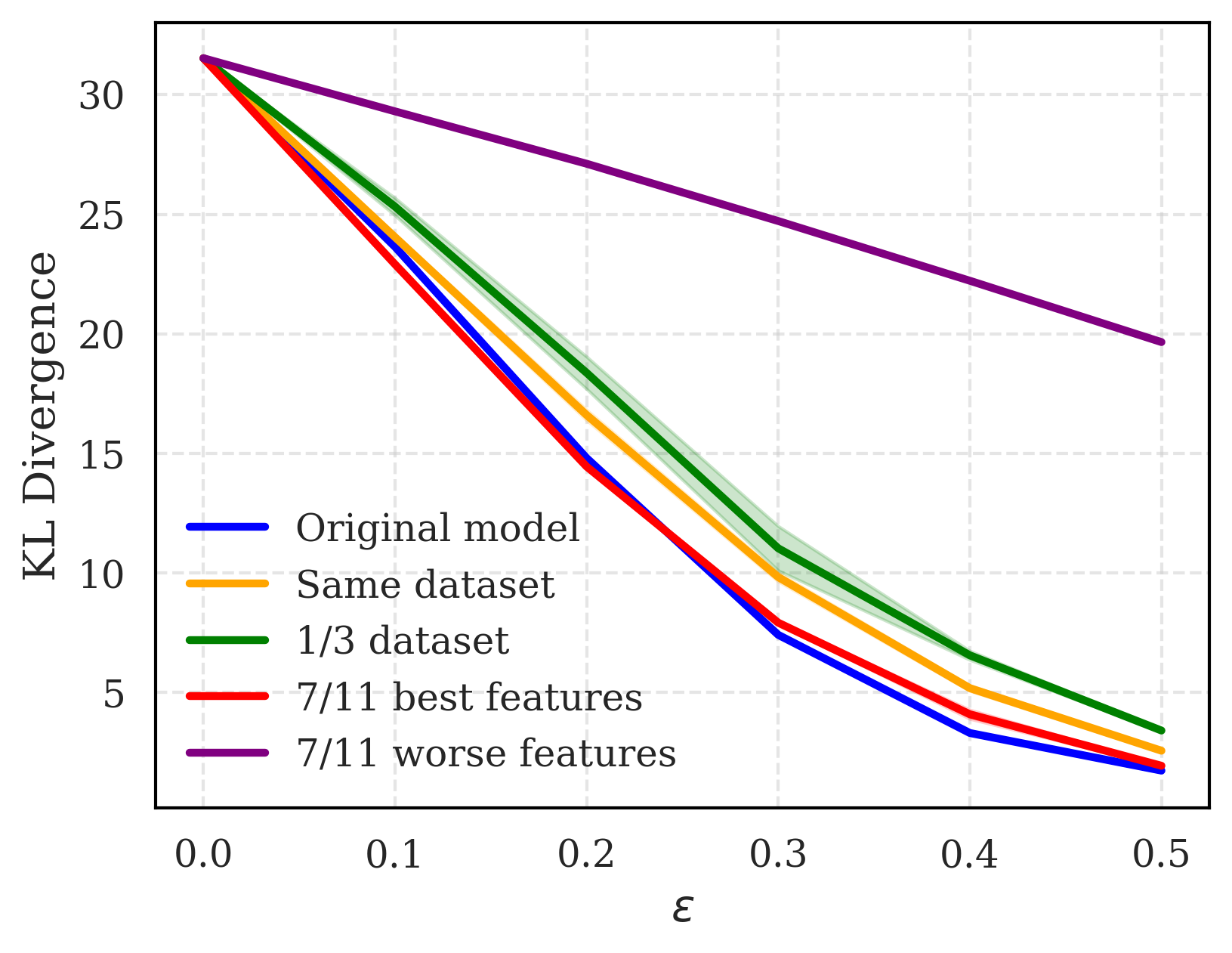}
    \subcaption{Attacks to full PPD.}\label{fig:kl_graybox}\par 
    \end{multicols}
\caption{SEP of gray-box attacks.}
\label{fig:wine_gray}
\end{figure}

We now examine gray-box attacks targeting the full PPD. Similar to  Section~\ref{sec:wine}, we focus on the case in which the defender aims to modify the wine indicators to shift the PPD towards a Gaussian with twice the mean and twice the standard deviation of the original PPD. Figure~\ref{fig:kl_graybox} represents the KL between the desired APPD and the distribution achieved by each attack variant. Notably, only the attack constrained to the 7 worst features shows significant performance degradation (KL = 19.63 at $\epsilon = 0.50$). The remaining gray-box scenarios demonstrate remarkable transferability, with the best-feature attack even outperforming the white-box baseline at small perturbation sizes. These results suggest that distribution-level attacks maintain effectiveness across various information constraints, as long as the adversary has access to quality features or representative data.

In summary, Figure \ref{fig:wine_gray} confirms that the performance degradation under gray-box constraints is modest, particularly when the attacker has access to either the full dataset or just the most informative features. These findings suggest that while gray-box constraints do reduce attack performance, practical adversaries with high-quality partial information can still mount effective attacks. Further discussion on gray-box attacks is available in the SM, Sections E and F.3.

\section{CONCLUSIONS}

We introduced two novel families of evasion attacks against Bayesian predictive models targeting both point predictions and the full PPD, 
 mainly in a white-box setting. The attacks apply broadly to classification and regression, requiring only sampling access to the PPD. Crucially, they can be applied to directly manipulate predictive uncertainty, an aspect largely overlooked in prior AML research. Our experiments reveal that even small perturbations can undermine predictive uncertainty, severely compromising OOD detection and exposing significant vulnerabilities in VI-based BNNs.

While our attacks highlight potential mitigation strategies, such as reducing collinearity, protecting key covariates, and leveraging security through obscurity (Section \ref{app:gray_toy} of SM), these remain only partial solutions. A security-by-design approach is essential, and our attacks offer a foundation to improve model robustness, particularly when integrated with probabilistic defensive approaches to AML \citep{riosInsua2023,gallego2024protecting}.

In summary, the proposed attacks serve primarily as diagnostic tools, identifying vulnerabilities in Bayesian  predictive systems. By adopting the defender’s perspective, our analysis aims to anticipate potential threats and inform robust, principled security improvements. Ethical responsibility must guide the deployment and use of deep learning models, especially within safety-critical applications.


\FloatBarrier


\begin{acknowledgements} 
This research was supported by the Spanish Research Agency under the \textit{Proyectos de Generación de Conocimiento 2022} grants No. PID2022-137331OB-C33 and PID2021-124662OB-I00  the EOARD-AFOSR project RC2APD, GRANT  13324227; the European Union’s Horizon 2020 Research
and Innovation Programme under Grant Agreement No. 101021797 (STARLIGHT).
PGA is a staff member hired under the Generation D initiative, promoted by Red.es, an organisation attached to the Ministry for Digital Transformation and the Civil Service, for the attraction and retention of talent through grants and training contracts, financed by the Recovery, Transformation and Resilience Plan through the European Union's Next Generation funds.

\end{acknowledgements}

\bibliography{biblio}

\onecolumn

\title{Evasion Attacks Against Bayesian Predictive Models\\(SM)}
\maketitle

\appendix

\section{PROOFS}
\subsection{Proof of Proposition 1}\label{app:proof1}
\begin{proof}
    First, applying the chain rule to $\|\,\cdot\,\|_2^2$, we have
    \begin{equation*}
       \nabla_{x'} J(x')
       ~=~
       2\,\bigl(\mu(x') - G^*\bigr)^\top 
       \,\nabla_{x'} \mu(x').
    \end{equation*}
    Under the assumptions, the Dominated Convergence Theorem (DCT) applies and the derivative passes inside the expectation. Finally, using the standard log-derivative trick \citep{mohamed2020monte} $\nabla_{x'} \mu(x')$, can be expressed as in \eqref{eq:point_att}.
\end{proof}

\subsection{Proof of Proposition 2}\label{app:proof2}
\begin{proof}
By the DCT, we can pass $\nabla_{x'}$ within the outer expectation and the gradient can be expressed as
    \begin{align*}
        &- \mathbb{E}_y \left[ \nabla_{x'} \log \pi(y | x', \mathcal{D} )\right] = - \mathbb{E}_y \left[ \frac{\mathbb{E}_{\gamma | \mathcal{D}} \left[ \nabla_{x'} \pi(y | x', \gamma )\right] }{\mathbb{E}_{\gamma | \mathcal{D}} \left[ \pi(y | x', \gamma )\right]}\right].
    \end{align*}
In the last step, we use again DCT to pass $\nabla_{x'}$ inside the expectation over $\gamma$.    
\end{proof}

\section{BAYESIAN LINEAR REGRESSION UNDER A NORMAL INVERSE GAMMA PRIOR}\label{app:normalinversegamma}

This Section details the specific Bayesian linear regression setup used in several of the experiments and the earlier analytical results. Consider the conjugate case of a linear regression model with normal-inverse-gamma priors 
\begin{equation}
    \beta \,|\, \sigma^2 \sim \mathcal{N}(\mu_0, \sigma^2 \Lambda_0^{-1}),  \,\,
    \sigma^2 \sim   \text{Inv-Gamma}(a_0, b_0) .
     \label{eq:normal_inverse_setting}
\end{equation}
Given $\mathcal{D}$, the posterior distributions for $\beta$ and $\sigma^2$ follow a normal-inverse-gamma distribution
with parameters \citep{smith1973general}
\begin{equation*}
        \mu_n = \Lambda_n^{-1} (\Lambda_0 \mu_0 + X^\top y), 
        \,\,  \Lambda_n = \Lambda_0 + X^\top X ,
        \end{equation*}
        \begin{equation*}
            a_n = a_0 + \frac{n}{2}, \,\,
    b_n = b_0 + \frac{1}{2} (y^\top y + \mu_0^\top \Lambda_0 \mu_0 - \mu_n^\top \Lambda_n \mu_n).
\end{equation*}
%
Then, the PPD for a new feature vector $x$, the corresponding response $y$ is a $t$-distribution with $2a_n$ degrees of freedom, mean $x^\top \mu_n$, and scale parameter $\frac{b_n}{a_n} \left( 1 + x^\top \Lambda_n^{-1} x \right)$. 

Notice that when $ \mu_0 = 0 $ and as $ \Lambda_0 \to 0 $, the maximum a posteriori (MAP) for $\beta | \sigma^2$ approaches the standard least squares estimator.
Besides, when $ \mu_0 = 0 $ and $ \Lambda_0 = cI $, the MAP 
will be the ridge regression estimator \citep{bishop2006pattern}.

\section{ANALYTICAL CONSIDERATIONS ON ATTACKS}

This Section covers several aspects concerning analytical versions of the
proposed attacks.

\subsection{Point prediction attacks}\label{app:point_analytical}

When analytical expressions for the posterior and the objective function exist, we employ general-purpose optimizers to approximate the attack \citep{powell2019unified}. Whether this is possible depends on the function $g$, and the geometry of $\mathcal{X}$.

A particularly common scenario arises when the adversary seeks to subtly modify covariates $x$ to steer the posterior predictive mean towards a value $y^*$. This corresponds to setting $g(x', y) = y$ 
  and $G^* = \{ y^* \}$. Consider a Bayesian linear regression model with normal-inverse-gamma prior. In this case, the PPD for a given $x$  follows a $t$-distribution with mean $ x^\top \mu_n $, where $ \mu_n $ depends on the prior parameters and observed data $ \mathcal{D} $.
Thus we can formulate the optimization problem
\begin{equation*}
    \min_{x'} \, | \mu_n^\top x' - y^* | \quad \text{s.t.} \quad \Vert x' - x \Vert \leq \epsilon.
\end{equation*}
Defining the shift $r = x' - x$, the problem is reformulated as
\begin{eqnarray*}
    \min_{r}& |\mu_n^\top r - \alpha| \quad \text{s.t} \quad \Vert r \Vert \leq \epsilon,
\end{eqnarray*}
where $\alpha = y^* - \mu_n^\top x$. 

Specific solutions depend on the 
chosen constraining norm. For instance, under the $L_2$ norm, Hölder's inequality implies that $|\mu_n^\top r| \leq \Vert \mu_n \Vert_2 \cdot \Vert r \Vert_2  \leq \Vert \mu_n \Vert_2 \epsilon$: the adversarial posterior mean will be achievable whenever $|\alpha| \leq \epsilon \Vert \mu_n \Vert_2$, and the optimal  adversarial data perturbation will be
\begin{equation}\label{mean_att_l2}
   r^* = \frac{\alpha}{\Vert \mu_n \Vert_2^2}\mu_n \Rightarrow x' = x +  \frac{y^* - \mu_n^\top x}{\Vert \mu_n \Vert_2^2}\mu_n.
\end{equation}
In the other case, 
$r$ should be set to the boundary of the ball in the direction of $\mu_n$,
  so that 
\begin{equation*}
   r^* = \text{sign}(\alpha)\frac{\epsilon}{\Vert \mu_n \Vert_2}\mu_n \Rightarrow x' = x +  \text{sgn}(y^* - \mu_n^\top x) \frac{\epsilon}{\Vert \mu_n \Vert_2}\mu_n.
\end{equation*}
This means that an adversarial posterior predictive mean is achievable under a perturbation $x'$ of $x$ such that $\Vert x' - x \Vert_2 \leq \epsilon$ whenever $|y^* - \mu_n^\top x| \leq \epsilon \Vert \mu_n \Vert_2$.

Similarly, under the $L_\infty$ norm constraint, Hölder's inequality
  entails $|\mu_n^\top r| \leq \Vert \mu_n \Vert_1 \epsilon$. Thus, if $|\alpha| \leq \Vert \mu_n \Vert_1 \epsilon$, the adversarial posterior mean is achievable with 
\begin{equation*}
   r^* =  \frac{\alpha}{\Vert \mu_n \Vert_1} \text{sgn}(\mu_n) \Rightarrow x' = x +  \frac{y^* - \mu_n^\top x}{\Vert \mu_n \Vert_1} \text{sgn}(\mu_n) ,
\end{equation*}
else we make
\begin{eqnarray*}
   r^* &=&   \epsilon  \cdot \text{sgn}(\mu_n) \cdot \text{sgn}(\alpha) \Rightarrow \\
   x' &=& x + \epsilon  \cdot \text{sgn}(\mu_n) \cdot \text{sgn}(y^* - \mu_n^\top x).
\end{eqnarray*}

\subsection{FULL PPD ATTACKS}\label{app:distr_analytical}

In order to have a closed-form expression for the objective function in \eqref{eq:adv_problem_2} a necessary condition is to have an analytical PPD. However, even if this is the case, the expectation required to compute the KL divergence might not have an analytical expression. For instance, in a Bayesian linear regression model with normal-inverse-gamma prior, the PPD is a Student's $t$-distribution. If we consider the APPD to be within the same family but with different location and/or scale parameters, we will not have a closed-form expression 
  as the KL divergence between two $t$-distributions lacks an analytical form. 

Should an analytical expression for the KL divergence in \eqref{eq:adv_problem_2}  be available, general-purpose optimizers \citep{powell2019unified} could be employed to approximate the attack. The choice of optimizer
will depend on the geometry of the objective function. In most cases, the 
 required KL divergence will not be convex in $x'$, necessitating the use of e.g.\ PGD, sequential quadratic programming, or interior point methods \citep{gondzio_InteriorPointMethods_2012}. 
As an example, in the Bayesian linear regression setting with known variance $\sigma^2$,  under a normal prior $\mathcal{N}(\mu_0, \Lambda_0^{-1})$ on the $\beta$ coefficients, the posterior predictive distribution $y | x, X, y$ is normal with mean $x^\top \mu_n$ and variance $x^\top \Lambda_n^{-1} x + \sigma^2$, where
%
%
 $  \Lambda_n = \Lambda_0 + \frac{1}{\sigma^2} X^\top X$ and
 $   \mu_n = \Lambda_n^{-1} \left(\Lambda_0 \mu_0 + \frac{1}{\sigma^2} X^\top y \right)$.
For a normal APPD with mean $\mu_A$ and variance $\sigma_A^2$, the KL divergence \eqref{eq:adv_problem_2} is 
\begin{align*}
    \frac{1}{2} \log \left( \frac{x'^T \Lambda_n^{-1} x' + \sigma^2}{\sigma_A^2} \right) + \frac{\sigma_A^2 + (\mu_A - x'^T \mu_n)^2}{2(x'^T \Lambda_n^{-1} x' + \sigma^2)} - \frac{1}{2}.
\end{align*}
The second term is the quotient of two quadratic forms in $x$, which is generally non-convex, and constrained non-convex optimization techniques will be required to find the optimal solution.

\section{COMPUTATIONAL COST OF ALGORITHMS 1 AND 2}\label{app:comp}

In both Algorithms 1 and 2, the computational time per iteration of stochastic gradient descent is largely driven by the number of posterior samples used, especially when these samples are generated via MCMC.
\begin{itemize}
\item For Algorithm 1, this number of samples is predetermined and fixed at $N + M$.

\item For Algorithm 2, the number of posterior samples drawn depends on the index selected at each iteration. Let $R$ denote the number of indices selected per iteration. If index $\ell$ is chosen, $M_\ell$ posterior samples need to be drawn, where $M_\ell = M_0 \cdot 2^\ell$. The probability of selecting index $\ell$ is $ \omega_\ell $, with weights $ \omega_\ell \propto 2^{-\tau \ell} $. Thus, the expected number of posterior samples for one index selection is $\sum_{\ell} \omega_\ell M_\ell $. Since $R$ indices are selected per iteration, the total expected number of posterior samples per iteration is: 
\[
\mathbb{E}[\text{Number of posterior samples per iteration}] = R \cdot \sum_{\ell} \omega_\ell M_\ell.
\]
%
Substituting the expressions of $ M_\ell $ and $ \omega_\ell $, we have:
\[
\mathbb{E}[\text{Number of posterior samples per iteration}] = R \cdot \sum_{\ell} \left( C \cdot 2^{-\tau \ell} \right) \cdot M_0 2^\ell = R \cdot M_0 \cdot C \cdot \sum_{\ell} 2^{(1-\tau)\ell},
\]
where $ C $ is a normalization constant ensuring $ \sum_{\ell} \omega_\ell = 1 $.
Now, the geometric series $ \sum_{\ell} 2^{(1-\tau)\ell} $ converges if $ \tau > 1 $ to $\frac{1}{1 - 2^{-(\tau - 1)}}$.
Thus, the total expected number of posterior samples generated per iteration of Algorithm 2 is:

\[
\mathbb{E}[\text{Total samples}] = R \cdot M_0 \cdot \frac{C}{1 - 2^{-(\tau - 1)}}= R \cdot M_0 \cdot \frac{1-2^{-\tau}}{1 - 2^{-(\tau - 1)}}.
\]
\end{itemize}

\section{BAYESIAN ANALYSIS OF GRAY-BOX ATTACKS}\label{app:gray_theory}

By leveraging the Bayesian framework, our attacks naturally extend to the gray-box setting in a principled and effective way. A particularly elegant approach to modeling the attacker's uncertainty about the defender's system is to posit a set of candidate models $\mathcal{M}_1, \ldots, \mathcal{M}_K$, each weighted by a prior belief $\pi_1, \ldots, \pi_K$. This leads the attacker to compute a posterior predictive distribution by marginalizing over model uncertainty:

\[
\pi(y \mid x', D) = \sum_{i=1}^K \pi_i \int \pi(y \mid \gamma_i, x')\, \pi(\gamma_i \mid \mathcal{M}_i, D)\, d\gamma_i,
\]
where $\pi(\gamma_i \mid \mathcal{M}_i, D)$ denotes the posterior over the parameters of the $i$-th model. The  formulation closely mirrors the principles of Bayesian Model Averaging \citep{hoeting1999bayesian}, incorporating model uncertainty into the attack design. As a result, the attacks proposed in our paper can be seamlessly and rigorously extended to the gray-box scenario by substituting the original PPD with this model-averaged version.

Computationally, this extension requires only a modest additional step in the gradient estimation procedure, involving sampling a model according to its prior probability. Furthermore, if the attacker has access to outputs from the defender’s model, they can iteratively refine these model priors in a fully Bayesian manner, further enhancing attack effectiveness as more information becomes available.

Importantly, this approach generalizes beyond a finite set of models: by placing continuous priors over model parameters, the discrete sum naturally extends to an integral over the model space, allowing the attacker to consider an infinite family of candidate models.

\section{ADDITIONAL RESULTS FOR SECTION 5.1}

\subsection{UNBIASEDNESS OF GRADIENT ESTIMATES}\label{app:unb_toy}

In the synthetic data setting from Section 5.1, we consider a regression model with known variance. Under a conjugate normal prior for the regression coefficients, the PPD has analytical form, following a $t$ distribution, Section B. Consequently, for both types of attacks, we can compute the corresponding objective functions and their gradients analytically. This allows us to empirically assess the unbiasedness of the gradient estimates for attacks on point predictions and full 
 PPDs produced by Algorithms 1 and 2, respectively. Figures \ref{subfig:point_gradient_estim_indep} and \ref{subfig:distr_gradient_estim_indep} show histograms of the gradient coordinate samples for both attack types, with red vertical lines indicating the actual gradient component values. These plots suggest that the gradient estimates produced by Algorithms 1 and 2 are indeed unbiased.

\begin{figure}[ht]
\begin{multicols}{2}
    \includegraphics[width=\linewidth]{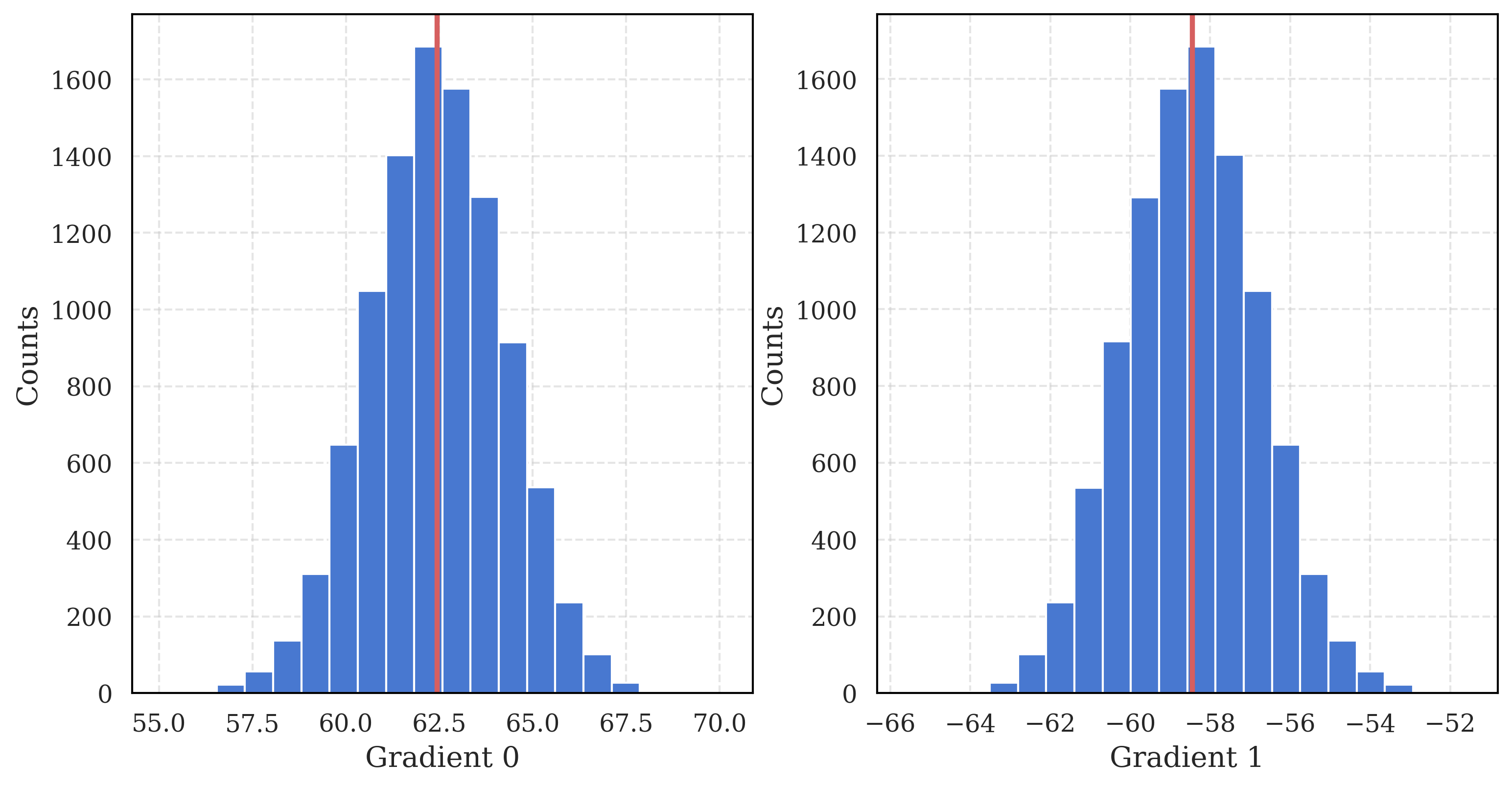}
    \subcaption{Attacks targeting point predictions.}\label{subfig:point_gradient_estim_indep}\par 
    \includegraphics[width=\linewidth]{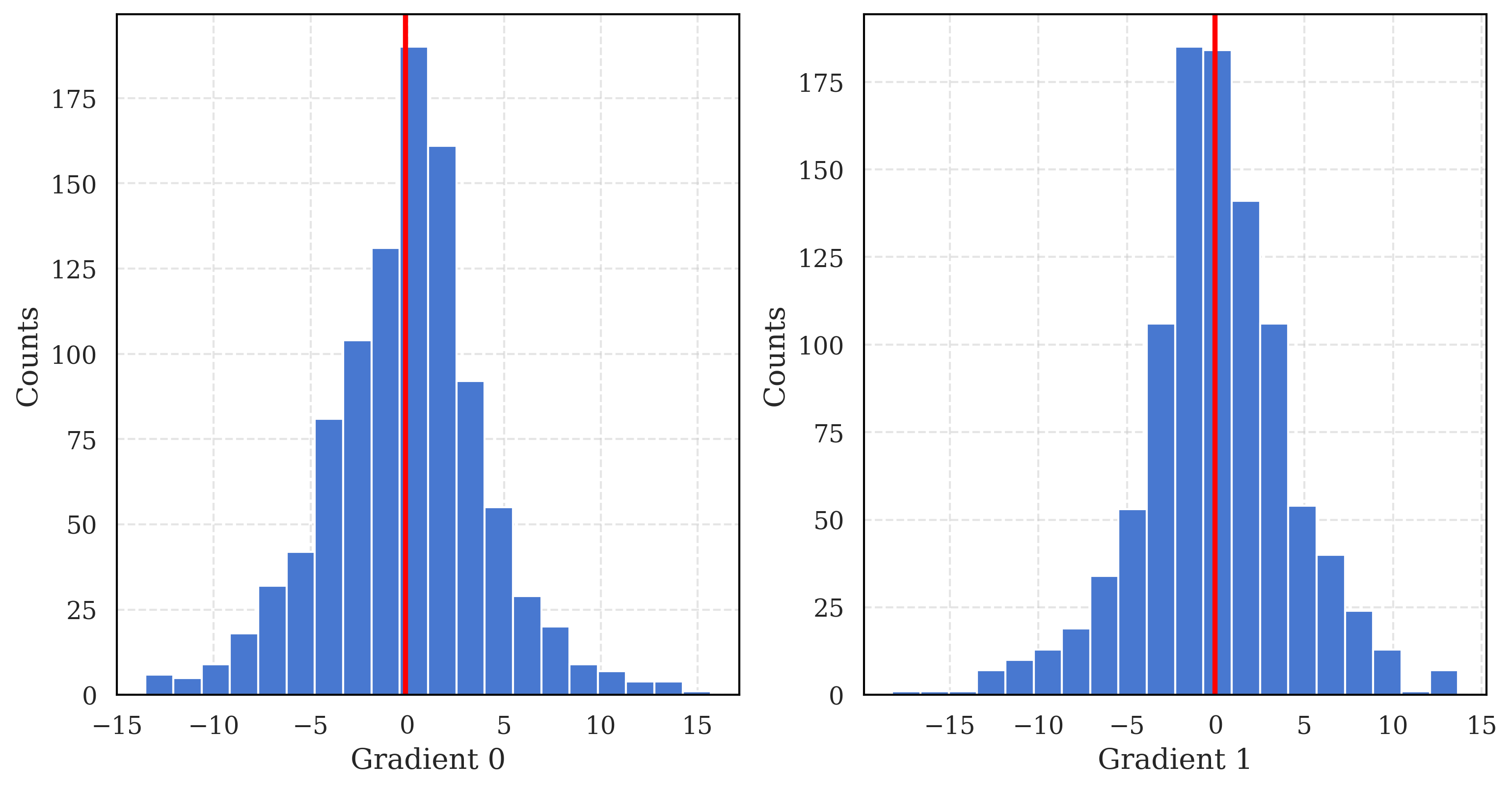}
    \subcaption{Attacks targeting full PPD.}\label{subfig:distr_gradient_estim_indep}\par 
    \end{multicols}
\caption{Estimated gradient distribution and true value for the synthetic dataset.}
\end{figure}

\FloatBarrier
\subsection{IMPACT OF COLLINEARITY ON ATTACKS TARGETING PREDICTIVE UNCERTAINTY}\label{app:coll_toy}

Section 5.1 demonstrated that, with large sample sizes, predictive uncertainty is primarily driven by aleatoric uncertainty, since epistemic uncertainty diminishes as the dataset grows. Since aleatoric uncertainty cannot be altered by manipulating covariates, attacks targeting predictive uncertainty are generally ineffective in such cases. However, in this Section, we observe that even with a large sample size, successful attacks targeting predictive uncertainty can still occur when there is high collinearity between regression covariates.
 To empirically demonstrate this, we generated a dataset of $n=1000$ data points using the regression setting in Section 5.1. 
   However, 
 instead of generating independent covariates from standard normal distributions, we drew them from a multivariate normal distribution with mean $(0,0)$ and covariance matrix $A^\top A$, with 
$$
A=\begin{pmatrix}
1 & 2 \\
3 & 4
\end{pmatrix}
$$ 
This induces a correlation between covariates of $\simeq 0.98$. 


\begin{figure}[ht]
\begin{multicols}{3}
    \includegraphics[width=\linewidth]{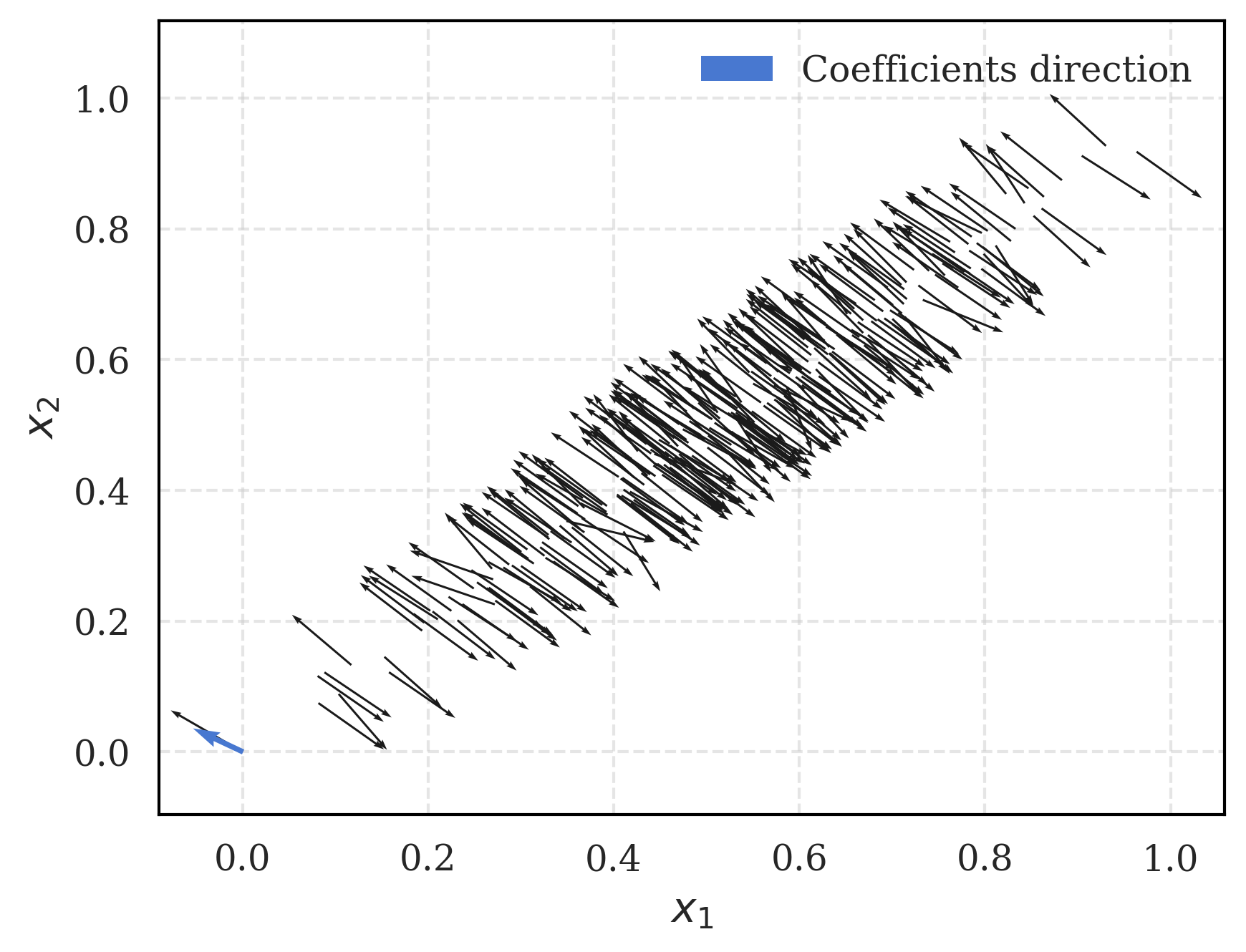}
    \subcaption{Perturbation field.}\label{subfig:distr_field_m2s}\par 
    \includegraphics[width=\linewidth]{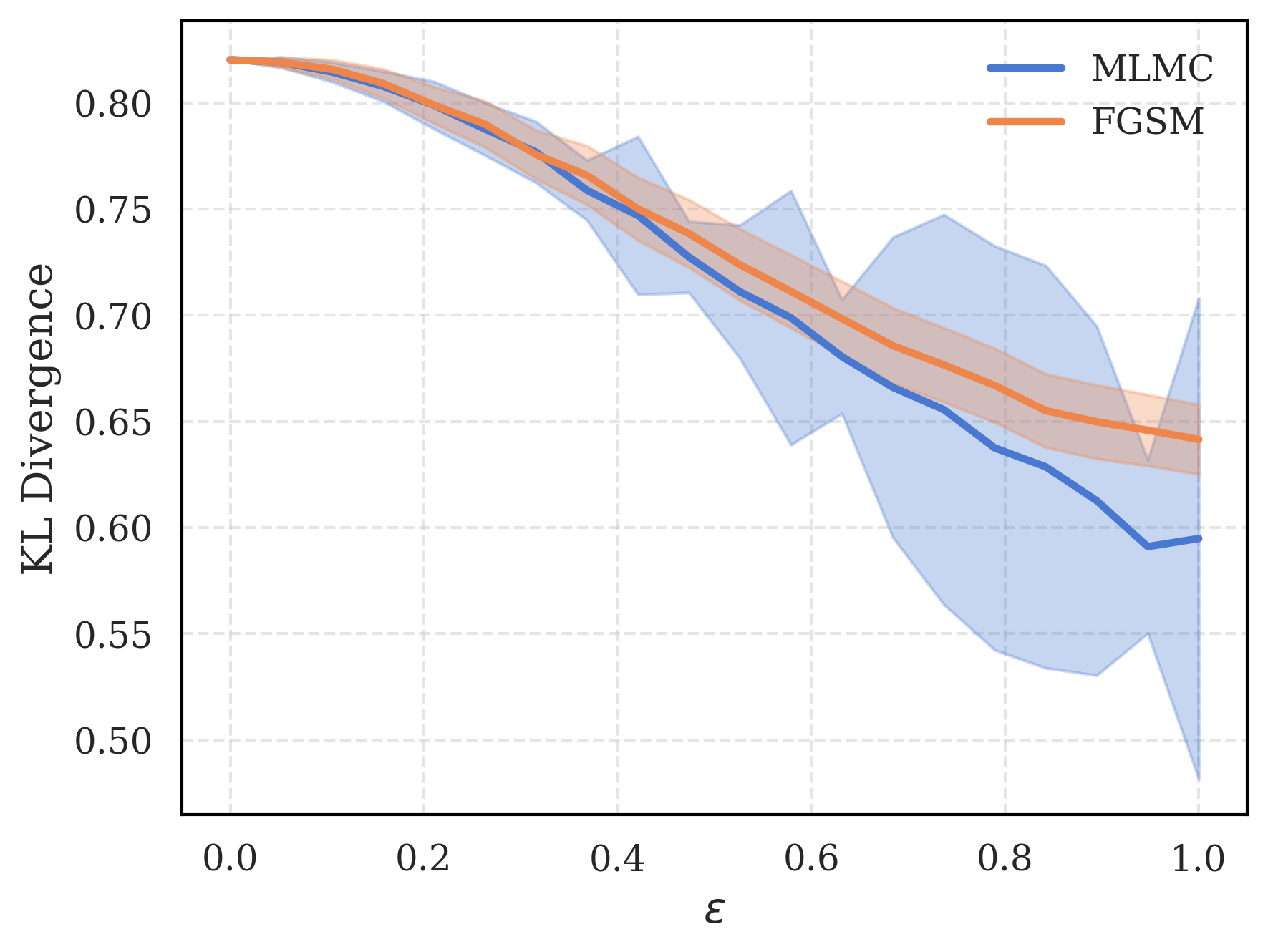}
    \subcaption{SEP.}\label{subfig:distr_m2s_eps}\par
    \includegraphics[width=\linewidth]{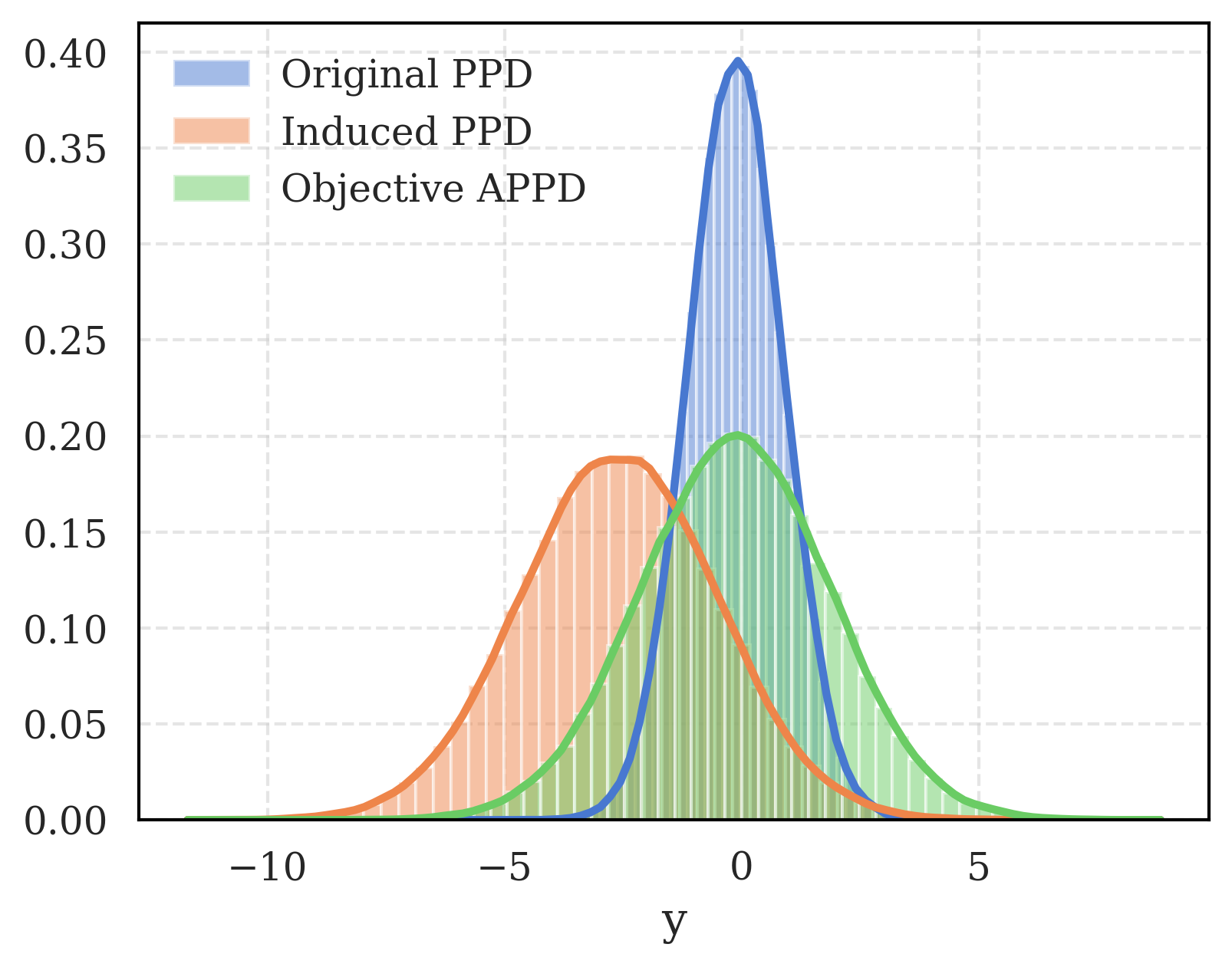}
    \subcaption{PPD attack example.}\label{subfig:distr_m2s}\par 
    \end{multicols}
\caption{Attacks targeting full PPD.}
\label{fig:sec}
\end{figure}

Figure \ref{fig:sec} evaluates the performance of an attack in this setting, where the goal is to steer the PPD of an instance with covariates $x$ towards an adversarial target with the same mean, but four times the variance. Figure \ref{subfig:distr_field_m2s} shows how test points are altered, with each attack represented by an arrow. The tail represents the unperturbed covariate vector, while the head indicates the perturbed one. Most instances shift in a direction parallel to the actual regression coefficients. Figure \ref{subfig:distr_m2s_eps} presents the KL divergence between the 
  APPD and the induced PPD for an initial point $(0.35, 0.35)$, across various attack intensities, comparing results from Algorithm 2 and FGSM. As shown, FGSM slightly underperforms compared to our method. Finally, Figure \ref{subfig:distr_m2s} compares the original, adversarial, and induced PPDs for an attack with $\epsilon = 2$. The attack increases the PPD variance but also causes an unintended shift in the mean. This shift may be compensated for by the increased variance in terms of KL divergence reduction. Notably, the same intensity attack performed on training data with uncorrelated covariates was unable to modify the PPD variance.

\subsection{GRAY-BOX ATTACKS}\label{app:gray_toy}

Throughout the paper, we have mainly operated in the white-box scenario, assuming that the attacker has full knowledge of the regression model, its prior, and its posterior. However, in most real-world applications, the attacker will not have complete information about these elements. This Section empirically evaluates the impact of this lack of knowledge on the effectiveness of attacks. Specifically, we revisit the regression setting from Section 5.1, but now the attacker lacks knowledge of the prior and the data used by the regressor. To compute the attacks, $A$ uses an alternative dataset $\mathcal{D}'$, which differs from $P$'s training set $\mathcal{D}$ but is generated from the actual model. Additionally, $A$ uses different priors for the parameters, specifically $\beta_0, \beta_1\sim \mathcal{N}(0, 1/2)$.

Figure \ref{subfig:gray_eps} presents the SEP for attacks targeting the mean of the PPD, with the same initial point and target as in Section 5.1, comparing the white-box and gray-box settings. While low-intensity attacks perform similarly in both cases, the gray-box attack significantly underperforms beyond an intensity of 0.3. This performance gap is even more pronounced for attacks targeting the full PPD. Figure \ref{subfig:gray_distr} illustrates the performance of full PPD attacks aimed at increasing posterior uncertainty (by setting the APPD to have four times the variance of the original PPD) while keeping the same mean. In the gray-box setting, no successful attacks were found, regardless of attack intensity. These experiments highlight the impact that {\em security through obscurity} measures could have on the performance of attacks.

\begin{figure}[ht]
\begin{multicols}{2}
    \includegraphics[width=\linewidth]{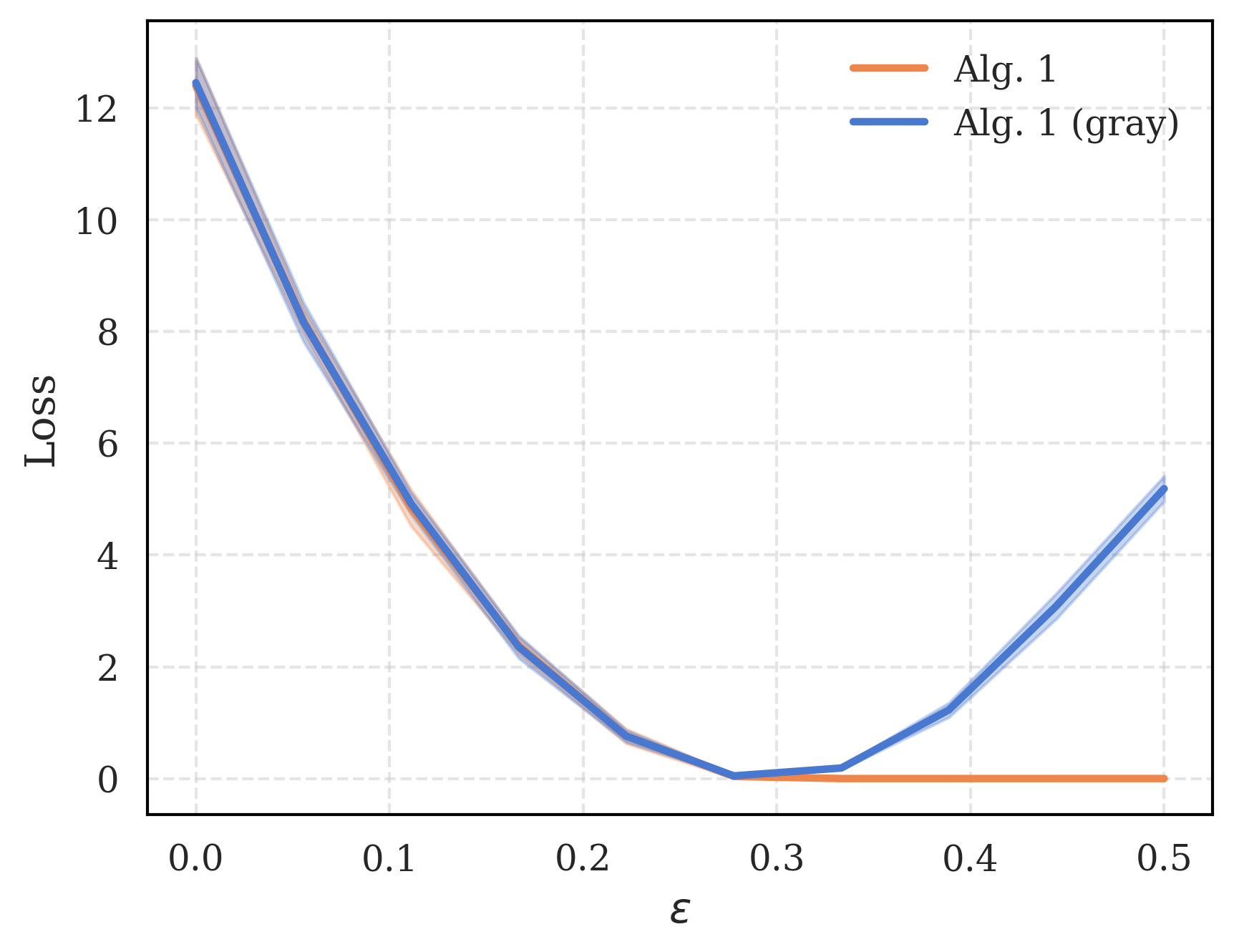}
    \subcaption{Attack targeting point predictions.}\label{subfig:gray_eps}\par 
    \includegraphics[width=\linewidth]{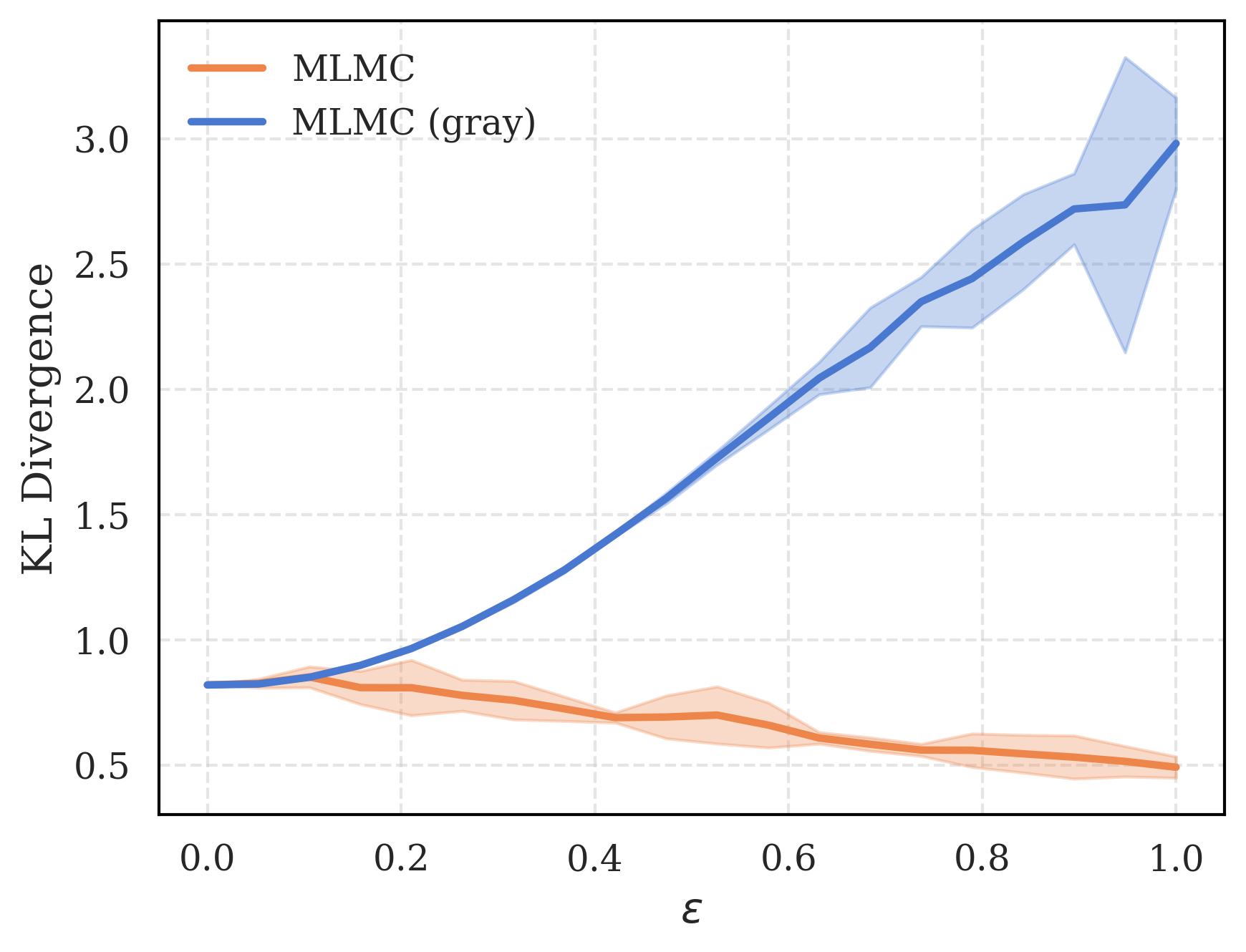}
    \subcaption{Attack targeting the full PPD.}\label{subfig:gray_distr}\par
    \end{multicols}
\caption{SEP for gray-box attacks.}
\label{fig:gray}
\end{figure}

\section{ADDITIONAL RESULTS FOR SECTION 5.2}\label{app:three}

Section 5.2 evaluated the effectiveness of attacks on three real-world datasets 
 using a linear regression model with a normal-inverse gamma prior. For point prediction attacks, the goal was to shift the PPD mean to twice the train set mean of the response variable, while full PPD attacks aimed to generate a PPD with double the mean and four times the variance of the original PPD.

For point prediction attacks, we presented the root mean squared deviation between the induced posterior predictive mean and its adversarial target on test instances across different attack intensities: 0.0 (no attack), 0.2, and 0.5, showing that higher attack intensities consistently move the induced mean closer to the adversarial target across all datasets. 
Table \ref{table:datasets_point_RMSE_R} instead shows the root mean squared error (RMSE) for each intensity and dataset, representing the deviation between the predictive mean under corrupted data and the true response on the test set. As shown, the attacks not only succeed in shifting the predictive mean toward the adversarial target but also degrade the predictor’s performance, with RMSE notably increasing as attack intensity grows.

\begin{table}[htbp]
\centering
\begin{tabular}{lccc}
\hline
\textbf{Dataset} & \textbf{0.0} & \textbf{0.2} & \textbf{0.5} \\
\hline
\textbf{Wine}    & $0.90 \pm 0.15$   & $2.04 \pm 0.27$  & $4.75 \pm 0.55$ \\
\textbf{Energy}  & $0.80 \pm 0.10$   & $1.38 \pm 0.16$  & $2.90 \pm 0.16$ \\
\textbf{Housing} & $0.80 \pm 0.08$   & $2.31 \pm 0.13$  & $2.37 \pm 0.12$ \\
\hline
\end{tabular}
\caption{Predictor's problem RMSE for attacks targeting point predictions against models on all datasets for different $\epsilon$ attack intensities.}
\label{table:datasets_point_RMSE_R}
\end{table}

Similarly, for attacks targeting the full PPD, Section 5.2 presented the KL divergence between the APPD and the PPD induced by the attack across all test instances, for each dataset and attack intensity. Higher attack intensities resulted in PPDs that more closely resembled the target distribution. In contrast, Table \ref{table:datasets_point_KL_R} shows the KL divergence between the induced PPD and the untainted PPD for each attack intensity and dataset, demonstrating that stronger attacks lead to distributions increasingly distant from the original one in KL divergence terms.

\begin{table}[htbp]
\centering
\begin{tabular}{lccc}
\hline
\textbf{Dataset} & \textbf{0.0} & \textbf{0.2} & \textbf{0.5} \\
\hline
\textbf{Wine}    & $0 \pm 0$   & $1.30 \pm 0.12$  & $7.55 \pm 0.43$ \\
\textbf{Energy}  & $0 \pm 0$   & $0.47 \pm 0.04$  & $2.98 \pm 0.34$ \\
\textbf{Housing} & $0 \pm 0$    & $2.85 \pm 0.33$   & $11.35 \pm 1.44$ \\
\hline
\end{tabular}
\caption{Predictor's problem KL for attacks targeting the full PPD against models on all datasets for different $\epsilon$ attack intensities.}
\label{table:datasets_point_KL_R}
\end{table}

\section{ADDITIONAL RESULTS FOR SECTION 5.3}\label{app:wine}

Figure \ref{fig:point_expy2_eps} presents the squared differences between the predictive utility of the wine under tainted indicators and the target utility, for several point attack intensities $ \epsilon $. Figure \ref{subfig:wine_PPD_sec} represents the estimated KL divergence between the objective PPD and the induced PPD for several attack intensities to the full PPD. 
\begin{figure}[h!]
\begin{multicols}{2}
    \includegraphics[width=0.8\linewidth]{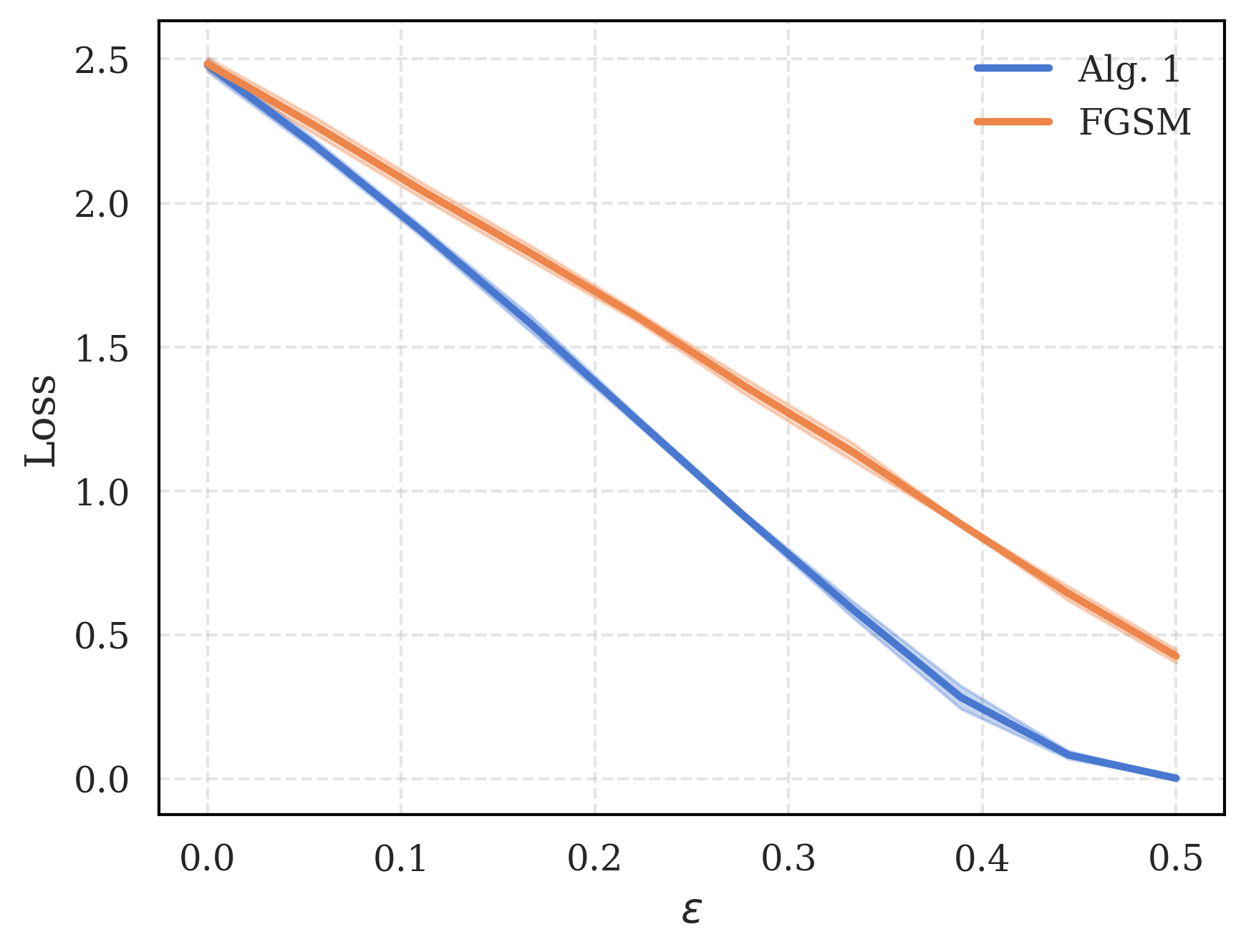}
    \subcaption{SEP of point attack.}\label{fig:point_expy2_eps}\par 
    \includegraphics[width=0.8\linewidth]{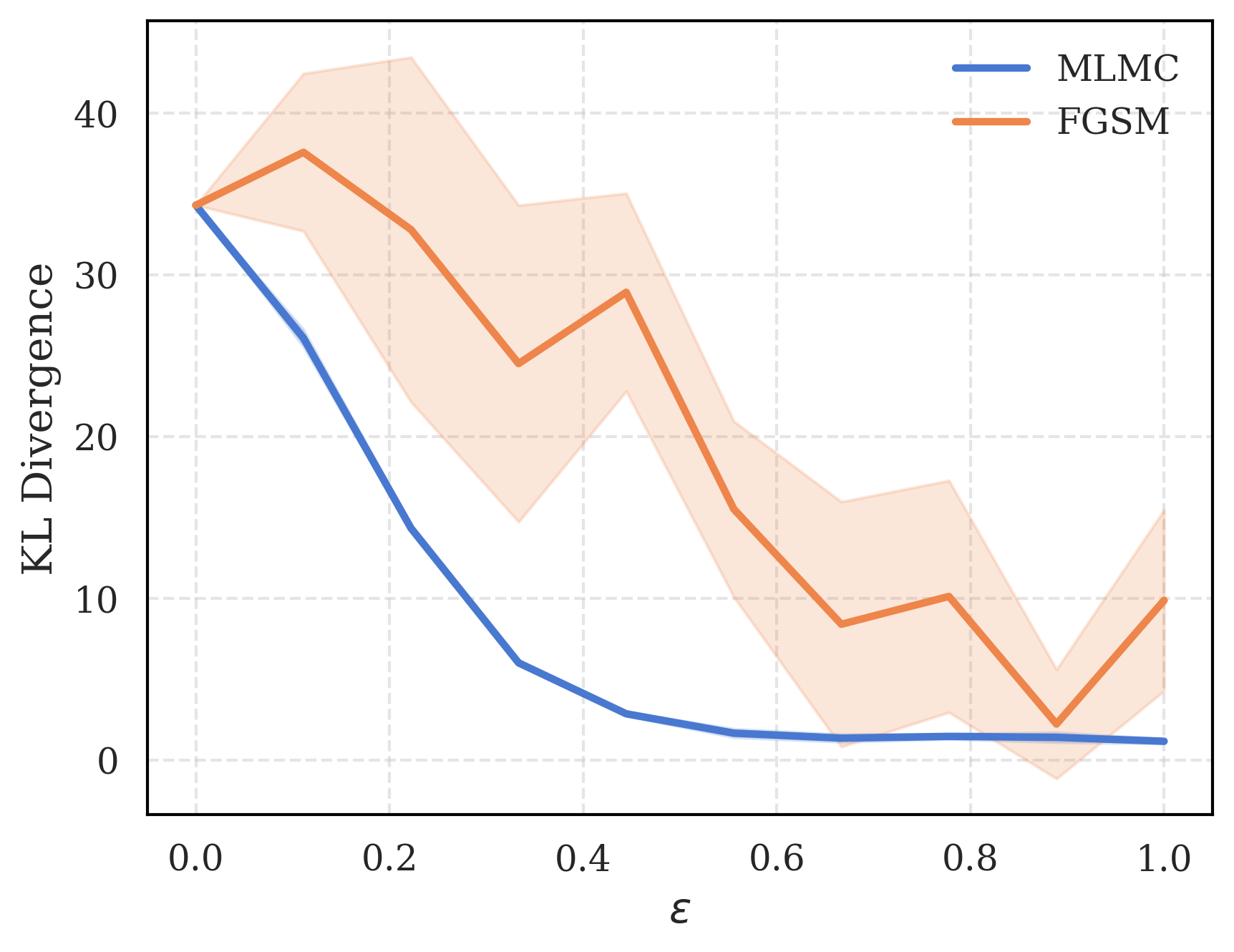}
    \subcaption{SEP of PPD attack.}\label{subfig:wine_PPD_sec}\par
    \end{multicols}
\caption{Attacks in the wine case-study.}
\label{fig:variance}
\end{figure}

Table \ref{tab:perturbations} presents the original and modified quality indicators found by our attack under $ L_2 $ and $ L_1 $ norm constraints with $\epsilon=0.3$. Attacks constrained by the $ L_1 $ norm result in sparser modifications compared to those with the $ L_2 $ norm. This sparsity is particularly relevant when the attacker has control over only certain covariates. The findings suggest that identifying which covariates influence more a specific predictive utility (i.e., those with non-zero perturbations) and prioritizing protection of these covariates from manipulation could serve as an effective defense.


\begin{table*}[h!]
    \centering
    \begin{tabular}{c|ccccccccccc}
        \hline
        \textbf{x} & 0.164 & 0.039 & 0.199 & 0.006 & 0.086 & 0.073 & 0.283 & 0.087 & 0.482 & 0.267 & 0.290 \\ \hline
        $L_1$-norm Attack & -0.006 & 0.017 & 0.000 & -0.081 & 0.008 & 0.000 & 0.000 & 0.176 & 0.000 & 0.000 & -0.012 \\ \hline
        $L_2$-norm Attack & -0.057 & 0.073 & -0.019 & -0.133 & 0.054 & -0.035 & 0.050 & 0.226 & -0.028 & -0.034 & -0.061 \\ \hline
    \end{tabular}
    \caption{Original and perturbed quality indicators under attacks with $L_1$ and $L_2$ Norm Constraints.}
    \label{tab:perturbations}
\end{table*}

\FloatBarrier
\section{ADDITIONAL RESULTS FOR SECTION 5.4}\label{app:mnist}

 Figure \ref{fig:mnist_seps} represents the PPD entropy for ID and OOD samples for several attack intensities under point and distribution attacks, respectively for the attacks given by Algorithms \ref{alg1} and \ref{alg2} (Figure \ref{fig:sep}) and their FGSM-inspired simplifications (Figure \ref{subfig:sep_fgsm}). Figure \ref{fig:sep} shows that both attacks present similar performance on OOD samples, while the distribution attack outperforms the point attack on ID samples.

\begin{figure}[h!]
\begin{multicols}{2}
    \includegraphics[width=0.8\linewidth]{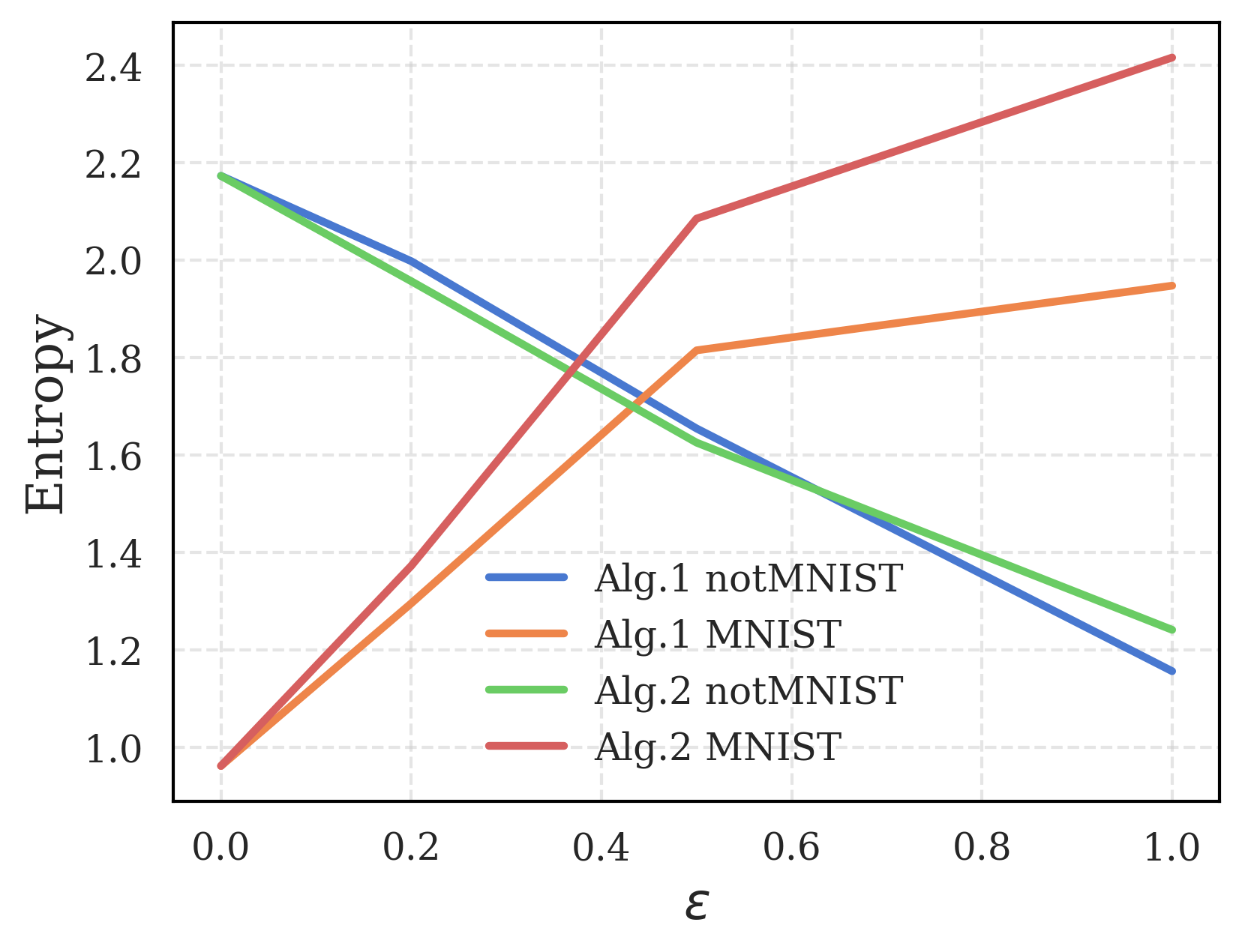}
    \subcaption{SEP of point and distribution attacks.}\label{fig:sep}\par 
    \includegraphics[width=0.8\linewidth]{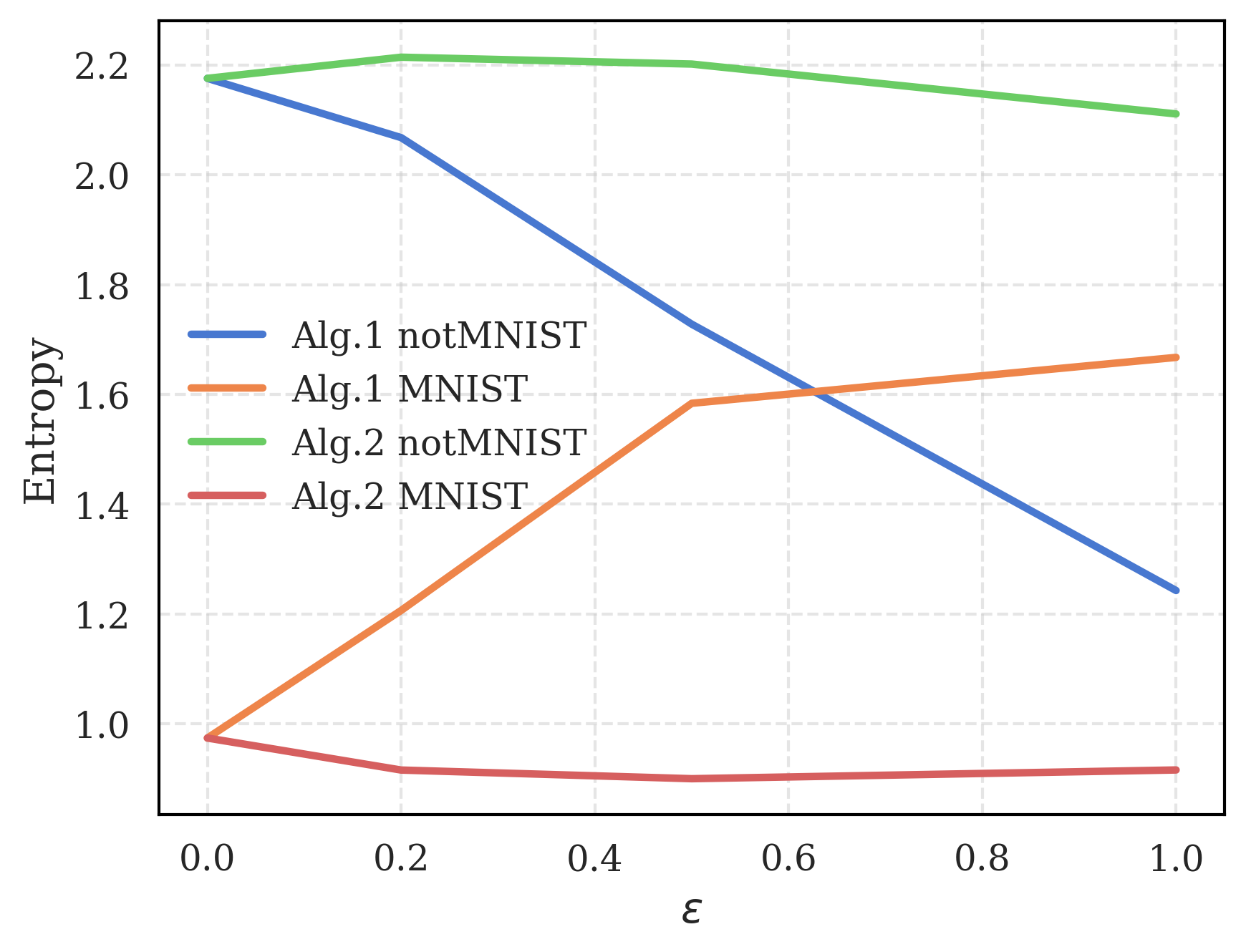}
    \subcaption{SEP of FGSM-inspired simplification of our attacks.}\label{subfig:sep_fgsm}\par
    \end{multicols}
\caption{Attacks in the MNIST case-study.}
\label{fig:mnist_seps}
\end{figure}

 Finally, we perform a point attack to the expected entropy of the deep ensemble model, assuming that samples are drawn from a Bayesian PPD. This requires minimal adaptation of Algorithm \ref{alg1}. Figure \ref{fig:point_entrop_de} presents  results.

\begin{figure}[ht]
\begin{multicols}{2}
    \includegraphics[width=0.95\linewidth]{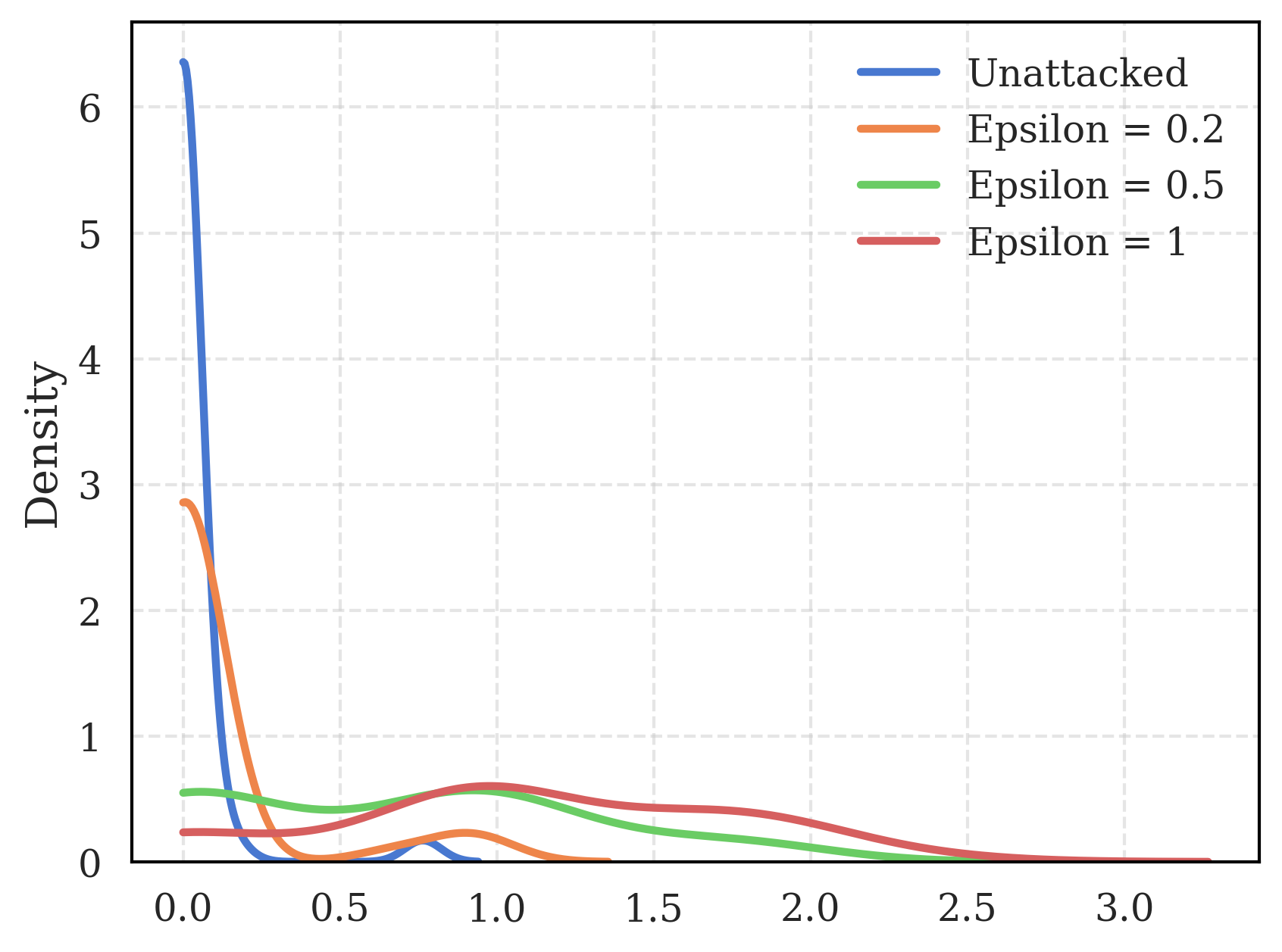}
    \subcaption{Inflating the entropy for MNIST data.}\label{fig:point_rise_de}\par  
    \includegraphics[width=0.95\linewidth]{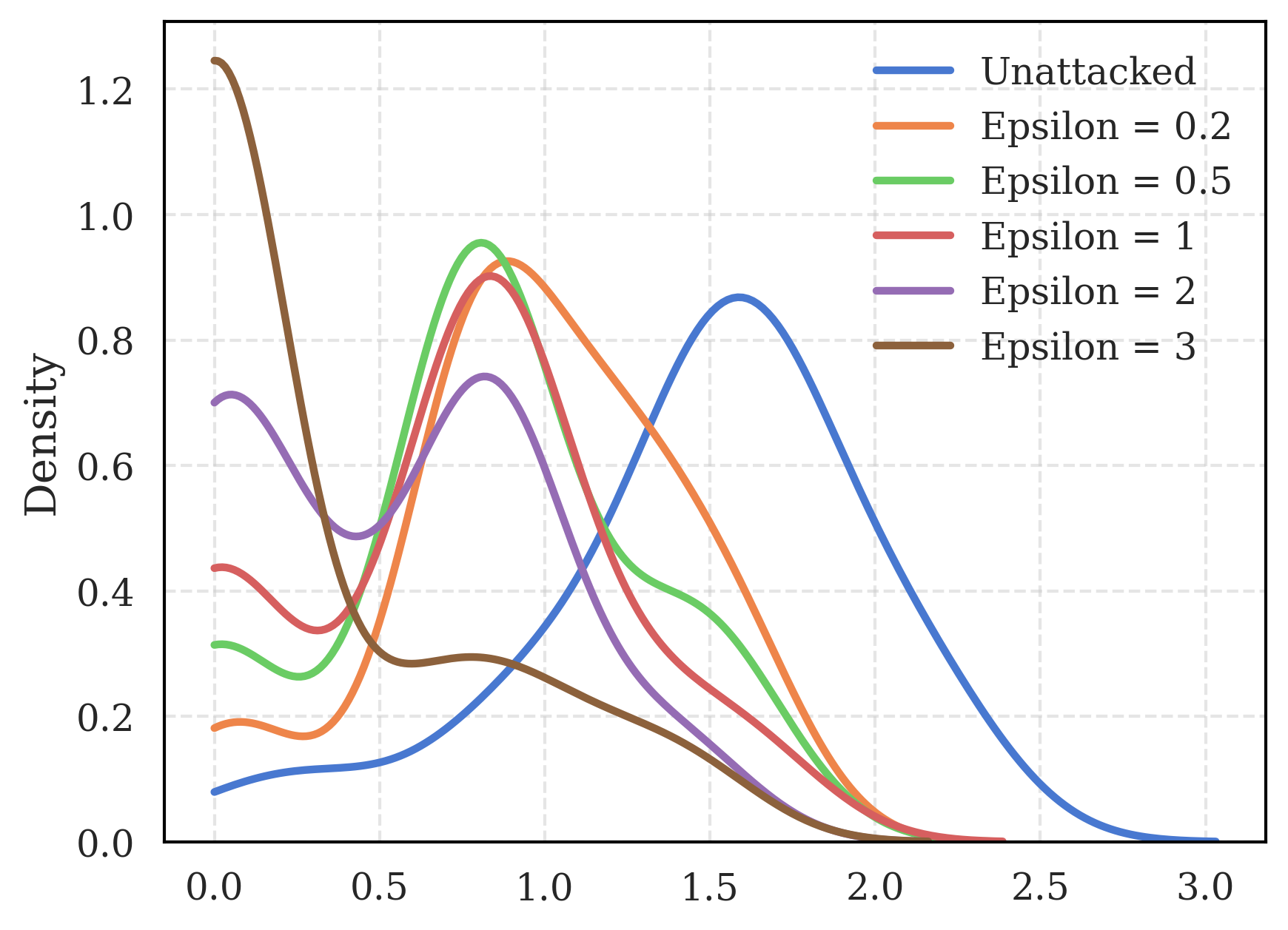}
    \subcaption{Lowering the entropy for notMNIST data.}\label{fig:point_lower_de}\par 
    \end{multicols}
\caption{Attacks to the expected entropy over the deep ensemble samples, assumed to be from the PPD.}
\label{fig:point_entrop_de}
\end{figure}

\end{document}